\documentclass{article}

\usepackage{microtype}
\usepackage{graphicx}
\usepackage{booktabs} 

\usepackage{hyperref}

\usepackage[accepted]{icml2024}

\usepackage{amsmath}
\usepackage{amssymb}
\usepackage{mathtools}
\usepackage{amsthm}
\usepackage[utf8]{inputenc} 
\usepackage[T1]{fontenc}    
\usepackage{hyperref}       
\usepackage{url}            
\usepackage{booktabs}       
\usepackage{amsfonts}       
\usepackage{nicefrac}       
\usepackage{microtype}      
\usepackage{thmtools,thm-restate}
\usepackage{tikz}
\usepackage{algorithm}
\usepackage{algorithmic}
\usepackage{bbm}
\usepackage{caption}
\usepackage{subcaption}
\usepackage{graphicx}
\usepackage{soul}
\usepackage{color-edits}
\usepackage{multirow}
\usepackage{footnote}
\usepackage{booktabs}

\usepackage{multibib}
\newcites{AP}{References for the Supplementary Material}

\usetikzlibrary{matrix,chains,positioning,decorations.pathreplacing,arrows}
\newcommand{\bmx}{x}
\newcommand{\bmy}{y}
\newcommand{\bmf}{f}
\newcommand{\bmz}{z}

\newcommand{\bmW}{W}
\newcommand{\calI}{\mathcal{I}}
\newcommand{\calJ}{\mathcal{J}}

\DeclareMathOperator*{\argmax}{arg\,max}

\addauthor{sd}{orange}
\addauthor{mb}{teal}
\addauthor{pb}{blue}
\addauthor{jw}{cyan}
\addauthor{ra}{red}
\addauthor{pf}{purple}

\usepackage[capitalize,noabbrev]{cleveref}

\theoremstyle{plain}

\newtheorem{lemma}{Lemma}

\theoremstyle{definition}

\theoremstyle{remark}

\usepackage[textsize=tiny]{todonotes}

\icmltitlerunning{Bayesian Optimization of Function Networks with Partial Evaluations}

\begin{document}

\twocolumn[
\icmltitle{Bayesian Optimization of Function Networks with Partial Evaluations}



\icmlsetsymbol{equal}{*}

\begin{icmlauthorlist}
\icmlauthor{Poompol Buathong}{equal,cam}
\icmlauthor{Jiayue Wan}{equal,orie}
\icmlauthor{Raul Astudillo}{caltech}\\
\icmlauthor{Samuel Daulton}{meta}
\icmlauthor{Maximilian Balandat}{meta}
\icmlauthor{Peter I. Frazier}{orie}
\end{icmlauthorlist}

\icmlaffiliation{cam}{Center for Applied Mathematics, Cornell University}
\icmlaffiliation{orie}{School of Operations Research and Information Engineering, Cornell University}
\icmlaffiliation{meta}{Meta}
\icmlaffiliation{caltech}{Department of Computing and Mathematical Sciences, Caltech}

\icmlcorrespondingauthor{Poompol Buathong}{pb482@cornell.edu}

\icmlkeywords{Machine Learning, ICML}

\vskip 0.3in
]



\printAffiliationsAndNotice{\icmlEqualContribution} 
\setcounter{footnote}{0} 

\begin{abstract}
Bayesian optimization is a powerful framework for optimizing functions that are expensive or time-consuming to evaluate. Recent work has considered Bayesian optimization of function networks (BOFN), where the objective function is given by a network of functions, each taking as input the output of previous nodes in the network as well as additional parameters. Leveraging this network structure has been shown to yield significant performance improvements. Existing BOFN algorithms for general-purpose networks evaluate the full network at each iteration. However, many real-world applications allow for evaluating nodes individually. 
To exploit this, we propose a novel knowledge gradient acquisition function that chooses which node and corresponding inputs to evaluate in a cost-aware manner, thereby reducing query costs by evaluating only on a part of the network at each step. 
We provide an efficient approach to optimizing our acquisition function and show that it outperforms existing BOFN methods and other benchmarks across several synthetic and real-world problems. Our acquisition function is the first to enable cost-aware optimization of a broad class of function networks.
\end{abstract}


\section{Introduction}
\label{sec:intro}
Bayesian optimization (BO) \citep{movckus1975bayesian,frazier2018tutorial} has emerged as a powerful framework for optimizing functions with expensive or time-consuming evaluations. BO has proved its efficacy in a variety of applications, including hyperparameter tuning of machine learning models \citep{snoek2012practical}, materials design \citep{frazier2008knowledge,zhang2020bayesian}, vaccine manufacturing \citep{rosa2022maximizing}, and pharmaceutical product development \citep{sano2020application}. 

In many applications, such as manufacturing \citep{ghasemi2018review}, epidemic model calibration \citep{garnett2002introduction}, machine learning pipeline optimization \citep{xin2021production}, and robotic control \citep{plappert2018multi}, objective functions are computed by evaluating a network of functions where each function
takes as input the outputs of its parent nodes. 
Consider the function network in Figure~\ref{fig:manuexam}, which illustrates the stages of a manufacturing process. The process begins with a raw material described by~$\bmx_1$. This raw material is used to produce an intermediate part described by~$\bmy_1$ through a process~$\bmf_1$.
Similarly, a second raw material described by~$\bmx_2$ is used to produce another intermediate part described by~$\bmy_2$ through a process~$\bmf_2$. 
These parts (with properties~$\bmy_1$ and~$\bmy_2$) are combined with another raw material described by~$\bmx_3$ in a process $\bmf_3$ to make the final product, the quality of which is denoted by $\bmy_3$.
Our goal is to choose $\bmx_1,\bmx_2,\bmx_3$ to maximize $\bmy_3$.
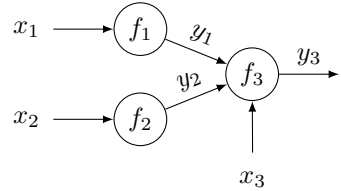
\begin{figure}[]
\centering
\begin{tikzpicture}[
init/.style={
  draw,
  circle,
  inner sep=0.7pt,
  minimum size=0.7cm
},
init2/.style={
  circle,
  inner sep=0.7pt,
  minimum size=0.7cm
},
]
\begin{scope}[start chain=1,node distance=8mm]
\node[on chain=1, init2]
(x1) {$\bmx_1$};
\node[on chain=1, init] 
  (f1) {$\bmf_1$};
\end{scope}
\begin{scope}[start chain=2,node distance=8mm]
\node[on chain=2,init] (f3) at (3,-6mm) {$\bmf_3$};
\node[on chain=2,init2] (f4) 
  {};
\end{scope}
\begin{scope}[start chain=3,node distance=8mm]
\node[on chain=3, init2] at (0,-12mm)
(x2) {$\bmx_2$};
\node[on chain=3,init]
 (f2) {$\bmf_2$};
 \node[on chain=3,init2] at (1.85,-20mm)
 (x3) {$\bmx_3$};
\end{scope}
\draw[-latex] (f1) -- (f3)node[pos=0.5,sloped,above] {$\bmy_1$};
\draw[-latex] (f2) -- (f3)node[pos=0.5,sloped,above] {$\bmy_2$};
\draw[-latex] (x1) -- (f1);
\draw[-latex] (x2) -- (f2);
\draw[-latex] (x3) -- (f3);
\draw[-latex] (f3) -- (f4)node[pos=0.5,sloped,above] {$\bmy_3$};
\end{tikzpicture}
\caption{An example function network in the manufacturing problem.}
\label{fig:manuexam}
\end{figure}

\citet{astudillo2021bayesian} showed that utilizing intermediate outputs in the network, i.e., $\bmy_1$ and $\bmy_2$, to decide which design parameters $\bmx=(\bmx_1,\bmx_2,\bmx_3)$ to evaluate significantly improves the performance of BO. However, this and other prior work have not exploited the ability to perform \emph{partial evaluations} of the function network, i.e., the ability to evaluate only a subset of nodes in the network at each iteration and use the so-obtained information to decide on the inputs to subsequent nodes, and potentially even to pause the evaluation process. As we demonstrate later, doing so can significantly improve performance, especially when evaluation costs vary significantly across nodes. For example, if evaluating~$\bmf_1$ is much cheaper than evaluating~$\bmf_2$, it may be advantageous to initially focus resources on understanding the range of values taken by $\bmy_1$ before performing too many costly evaluations of $\bmf_2$.

In this work, we introduce a BO algorithm that significantly improves performance over existing methods by taking advantage of the ability to perform partial evaluations.
This algorithm iteratively selects a node in the function network and a corresponding input to evaluate it, with the goal of identifying the global optimum within a limited budget.

Our contributions are summarized as follows:
\begin{enumerate}
    \item We introduce a framework for Bayesian optimization of function networks that allows partial evaluations.   
    \item We propose a knowledge-gradient-based acquisition function (p-KGFN) that, to our knowledge, is the first to actively leverage partial evaluations in general function networks in a cost-aware fashion.
    \item We propose an approximation of p-KGFN that can be optimized more efficiently.
    \item We demonstrate the benefits of exploiting partial evaluations through several numerical experiments, including both synthetic and real-world applications with a variety of network structures.
\end{enumerate}

\section{Related Work}
\paragraph{Grey-box BO} Our work falls within grey-box BO \citep{astudillo2021thinking}, which focuses on exploiting the known structure of the objective function (e.g., the function network structure considered in our work) to improve sampling decisions. BO of functions with a composite or network structure has been previously studied in the literature. For instance, \citet{uhrenholt2019efficient} considered objective functions that are sums of squared errors, while \citet{astudillo2019bayesian} and \citet{jain2023bayesian} considered a more general setting where the objective function is the composition of an expensive vector-valued inner function and a cheap outer function. BO of function networks was pioneered by \citet{astudillo2021bayesian}, introducing a probabilistic model that exploits function network structure and pairing this model with the expected improvement (EI) acquisition function \citep{jones1998efficient}. 

\paragraph{BO with Partial Evaluations} The ability to perform partial evaluations in the context of BO of function networks has been studied for specific network structures. \citet{kusakawa2022bayesian} considered function networks constituted by a chain of nodes and developed an algorithm that can pause an evaluation at an intermediate node. However, their approach, which uses an EI-based acquisition function, cannot be easily extended to quantify the value of evaluating a single node in more general function networks. \citet{lin2021bayesian} explored a setting where changing values of a subset of variables corresponding to different stages in a pipeline incurs a ``switching cost''. Their approach assumes fully sequential dependence between stages and cannot reuse previous evaluations. 
Additionally, \citet{lin2021bayesian} adopted a ``slow-moving bandit'' formulation that aims to minimize cumulative regret, whereas we seek to minimize simple regret.
Outside the function networks setting, \citet{hernandez2016} and \citet{daulton2023hvkg}  considered BO with partial evaluations for constrained and multi-objective optimization, respectively.

\paragraph{Cost-aware BO} Our work is related to research considering heterogeneous evaluation costs across the search space. Our approach is similar in nature to those proposed by \citet{snoek2012practical}, \citet{ wu2020practical}, and \citet{daulton2023hvkg}, whose acquisition functions value points based on the value of information per unit cost, thus favoring lower-cost evaluations. \citet{lee2020cost} adopted a cost-cooling schedule that discourages high-cost points early in the BO loop, \citet{abdolshah2019cost} incorporated cost-aware constraints while solving multi-objective BO problems, and \citet{astudillo2021multi} and \citet{lee2021nonmyopic} proposed non-myopic acquisition functions formulated using Markov decision processes for solving budgeted BO problems. 

\section{Problem Statement and Statistical Model}
\label{sec:BOFNwithPE}

\subsection{Problem Statement}
\label{subsec:BOFN_PE:problem}

Following the setup of \citet{astudillo2021bayesian}, we consider a sequence of functions $\bmf_1,\bmf_2,\hdots,\bmf_K$, arranged as nodes in a network representing the evaluation process. Specifically, the network structure is encoded as a directed acyclic graph  $\mathcal{G}=(\mathcal{V},\mathcal{E}),$ where $\mathcal{V}=\{1,2,\hdots,K\}$ and $\mathcal{E}=\{(i,j): \bmf_j \text{ takes the output of  } \bmf_i \text{ as input}\}$ denote the sets of nodes and edges, respectively. We assume that the final node function, $f_K$, is scalar-valued. However, the other node functions may be vector-valued. 

Let $\calJ(k)$ denote the parent nodes of node $k$.  Without loss of generality, we assume that nodes are ordered such that $j<k$ for all $j\in \calJ(k)$. Let $\calI(k)\subseteq\{1,2,\hdots,d\}$ denote the set of components of the input vector $\bmx\in\mathbb{X}\subset \mathbb{R}^d$ taken as an input by each function $\bmf_k$.\footnote{This set may be empty for some nodes if they take as input only the outputs from their parent nodes.}  The output of node $k$ when the function network is evaluated at $x$ is denoted by $y_k(x)$. The outputs $y_1(x), y_2(x),\hdots, y_K(x)$ can be computed recursively as
\begin{equation}
y_k(x)=f_k(y_{\calJ(k)}(x),x_{\calI(k)}), \  k=1,2,\hdots,K.  
\end{equation}

For each node $k$, we assume that there is an associated known positive evaluation cost function $c_k(\cdot)$.\footnote{When $c_k(\cdot)$ is unknown, we may learn it using a surrogate model and compute quantities involving costs by either taking the expectation over the distribution of $c_k(\cdot)$ or by replacing the cost function by the cost model's posterior mean.}  Our goal is to maximize the final node's function value~$\bmy_K(\bmx)$ while minimizing the cumulative evaluation cost. To support this goal, our algorithm will select at each iteration a node $k$ and corresponding input $z_k$ at which $f_k$ will be evaluated.  

We distinguish two settings associated with the feasible values of $z_k$:
\begin{enumerate}
    \item Evaluating a node $k$ requires to previously obtain the outputs from its parent nodes, in which case $z_k$ is comprised of the concatenation of these values and the additional parameters corresponding to node $k$.
    \item The possible outputs of each node are known, and each node $k$ can be evaluated at any feasible input (any admissible controllable input as well as any possible output of its parent nodes).
\end{enumerate}
We focus on the first setting, which aligns with many practical situations. For example, in our manufacturing problem, executing a step requires the outputs of the preceding steps. Additionally, we restrict our attention to function networks where pairs of nodes do not share common inputs. This ensures there are valid combinations for evaluation at downstream nodes. However, this assumption can be relaxed by grouping nodes with shared inputs as a preprocessing step.


Finally, we consider the scenario in which each intermediate output is reusable. In other words, once a node's output is obtained, it can be repeatedly used in downstream evaluations. This scenario is common in settings such as machine learning (ML) pipeline optimization, where trained ML models can be saved and reused, or in large-batch manufacturing, where the manufactured batch volume is effectively infinite relative to the amounts required downstream.

\subsection{Statistical Model}
\label{sec:statmodel}
Following \citet{astudillo2021bayesian}, we model the functions $\bmf_1,\bmf_2,\hdots,\bmf_K$ as samples from independent Gaussian process (GP) prior distributions \citep{williams2006gaussian}. For each $k=1,2,\hdots,K$, let $\mu_{0,k}$ and $\Sigma_{0,k}$ denote the prior mean and covariance functions of $\bmf_k$, respectively. Let $\mathcal{D}_{n,k} = \{(\bmz_{j,k},\bmy_{j,k})\}_{j=1}^{n_k}$ denote the observations at node $k$ after $n$ iterations, where $n_k$ is the number of observations at node $k$. The posterior distribution over $\bmf_k$ given $\mathcal{D}_{n,k}$ is a Gaussian process whose mean and covariance functions, denoted by $\mu_{n,k}$ and $\Sigma_{n,k}$, can be computed in closed form using the standard GP regression formulas \citep{williams2006gaussian}.

Let $\mathcal{D}_n = \{\mathcal{D}_{n,k}\}_{k=1}^K$ denote the observations at all nodes after $n$ iterations. The posterior distributions over $\bmf_1, \bmf_2, \hdots, \bmf_K$ given $\mathcal{D}_n$ induce a posterior distribution on the final node value $\bmy_K$. Although this distribution is generally non-Gaussian, we can obtain samples from it efficiently, as discussed in Section~\ref{sec:postest}.

Our acquisition function, formally defined in Section~\ref{sec:BOFN_PE:kgacqf}, is constructed based on these posterior distributions and evaluation costs $c_k(\cdot)$. It quantifies the cost-normalized benefit of performing one additional partial evaluation at a specific node. Our BO algorithm then decides to evaluate at a node $k^*$ with input $\bmz^*_{k^*}$ yielding the maximum value of this acquisition function.

\section{The p-KGFN Acquisition Function}
\label{sec:BOFN_PE:kgacqf}

Throughout this section, we assume that $n$ samples have already been observed and are determining how to allocate sample $n+1$.

Let $\nu_{n}(\bmx)$ denote the posterior mean of  $\bmy_K(\bmx)$ given $\mathcal{D}_n$.
Assuming risk-neutrality, the solution we would select if we were to stop at time $n$ would be an $\bmx$ that maximizes the posterior mean of the final node's value,\footnote{Note that we are concerned about the cost of evaluating a configuration \emph{during but not after} the optimization.} i.e., a solution of
\begin{equation}
    \nu^*_{n} = \max_{\bmx\in\mathbb{X}} \nu_{n}(\bmx).\label{eqn:mustar}
\end{equation}
Now, suppose one additional evaluation at a single node is allowed. For a node $k$ with a given input $\bmz_k$, observing $\bmf_{k}(\bmz_k)$ would result in an updated posterior over~$\bmf_k$, which in turn yields an updated posterior mean function of the final node value $\nu_{n+1}(\cdot)$ and also an updated maximum of the final node's posterior mean $\nu_{n+1}^{*}$. The difference between the two quantities,  i.e., $\nu^*_{n+1}-\nu^*_{n}$, quantifies the increment in the expected solution quality. 

We note that $\nu^*_{n+1}-\nu^*_{n}$ is random at time $n$ due to its dependence on the yet unobserved value of $\bmf_k(\bmz_k)$. Our acquisition function is obtained by taking the expectation of this increment with respect to the posterior on $\bmf_k(\bmz_k)$ and dividing it by the evaluation cost $c_k(\bmz_k)$. Specifically, we define the knowledge gradient for function networks with partial evaluations (p-KGFN) by
\begin{equation}
    \alpha_{n,k}(\bmz_k) = \frac{\mathbb{E}_{\bmy_k}[\nu_{n+1}^*]- \nu_{n}^*}{c_k(\bmz_k)}.
\label{eqn:kgfn}
\end{equation}

The feasible set for $z_k$ is given by $\mathbb{Z}_{n,k}:= \mathbb{Y}_{n,\calJ(k)}\times\mathbb{X}_{\calI(k)}$, where $\mathbb{Y}_{n,\calJ(k)}$ is the discrete set constituted by the outputs from the parent nodes of node $k$ that have been previously generated after $n$ iterations and $\mathbb{X}_{\calI(k)}$ is the set of possible additional parameters at node $k$. Thus, at each iteration, the next node and corresponding inputs to evaluate are given by  
\begin{equation}
    (k^*,\bmz^*_{k^*}) \in \argmax_{k\in\{1,\hdots,K\}, \  \bmz_k\in\mathbb{Z}_{n,k}}\  \alpha_{n,k}(\bmz_k). \label{eqn:kgdecision}
\end{equation}

Our acquisition function generalizes the classical knowledge gradient for regular BO \citep{frazier2008knowledge,wu2016parallel}.
Moreover, it is cost-ware (in that it favors lower-cost evaluations at the same expected quality) and thus is similar in nature to the acquisition functions proposed by \citet{snoek2012practical}, \citet{wu2020practical}, and \citet{daulton2023hvkg}.

\begin{figure*}[h!]
    \centering
    \includegraphics[height=14cm]{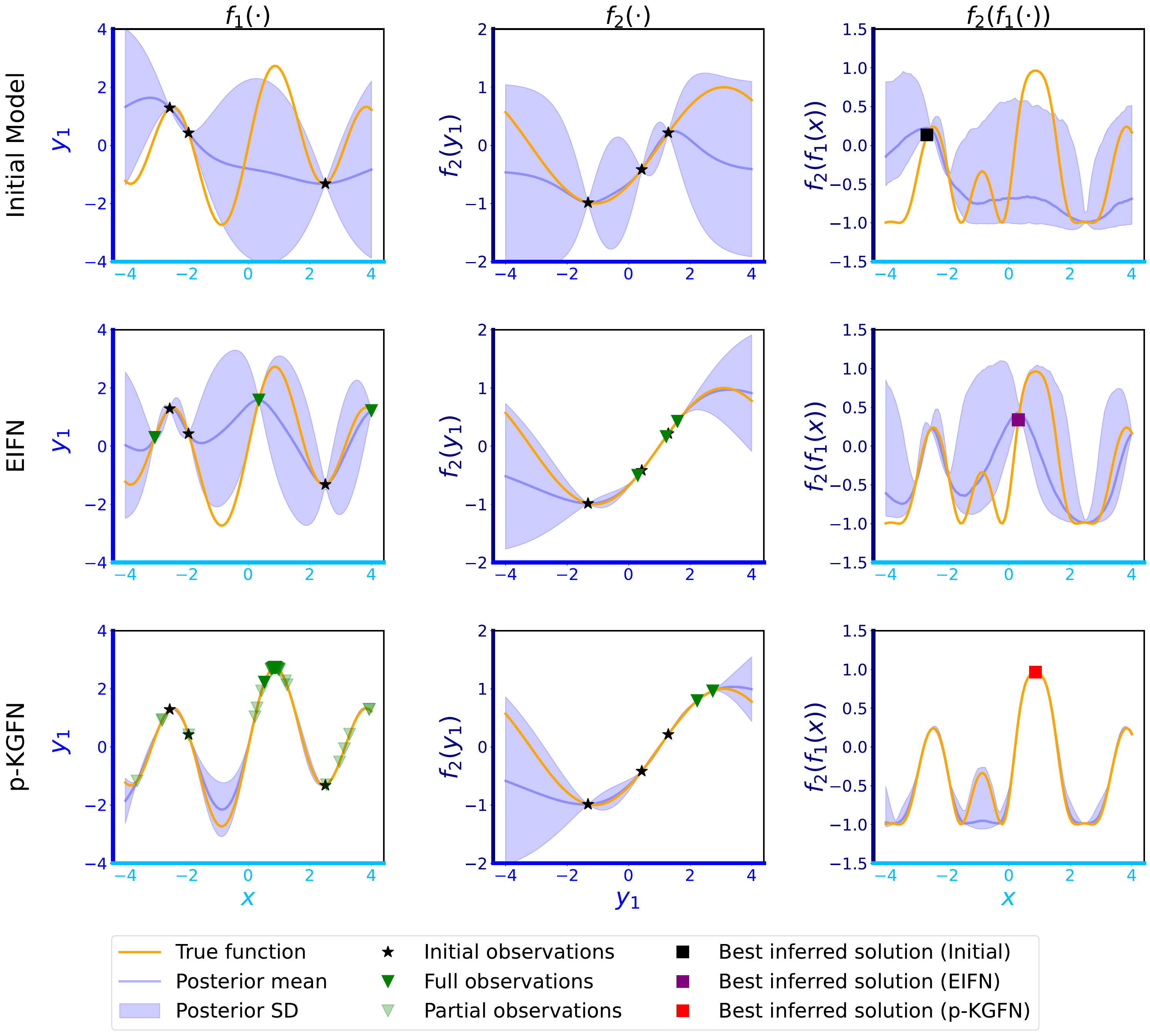}
    \caption{Comparison of EIFN and p-KGFN on a 1-D synthetic two-stage function network $f_2(f_1(\cdot))$. Top row (left to right): Initial models for $f_1(\cdot)$, $f_2(\cdot)$ and $f_2(f_1(\cdot))$. Second and third rows: Resulting surrogate models upon budget depletion by EIFN and p-KGFN, respectively. Each ground truth function is represented by an orange curve, while blue curves and shaded blue areas denote posterior mean and standard deviation, respectively. Black stars indicate the initial three points fully evaluated across the network for both algorithms. Dark green triangles represent the locations of full network evaluations. Light green triangles represent partial observations where only the first node was evaluated by p-KGFN. Black, purple, and red squares correspond to the initial and final inferred best solutions identified by EIFN and p-KGFN, respectively. We use the different colors for each axis to represent different types of inputs and outputs of the network as follows: light blue denotes the original input $x$ to the network, dark blue denotes the output of the first node $y_1$, and dark navy denotes the output of the second node $y_2$.}
    \label{fig:toyproblem}
\end{figure*}

\subsection{Advantages of Partial Evaluations}\label{subsec:advantage_partial}
In this section, we illustrate the benefits of performing partial evaluations, as enabled by p-KGFN, through a simple two-stage function network example. 
Consider $f_1(x) = \sin(x) + 2\sin(2x)$ with domain $x\in[-4,4]$, and $f_2(y) = \sin(3(y-1) / 4)$, which takes as input the output of $f_1$. Additionally, assume that evaluation costs are constant given by $c_1=1$ for the first stage and $c_2=49$ for the second stage. We analyze the behavior of our proposed acquisition function, p-KGFN, and the acquisition function proposed by \citet{astudillo2021bayesian},  EIFN, which also leverages the function network structure of the objective but requires full network evaluations at each iteration.

As shown in Figure~\ref{fig:toyproblem}, both EIFN and p-KGFN begin with three initial observations (black stars), evaluated across the full network. 
The initial models for $f_1(\cdot)$, $f_2(\cdot)$ and $f_2(f_1(\cdot))$ are presented in the first row. 
Both algorithms are allocated an evaluation budget of 150, which is equivalent to performing three evaluations of the full network.  
Rows two and three  
show the evaluations and resulting models upon budget depletion using EIFN and p-KGFN, respectively.
We observe that EIFN makes decisions aimed at identifying the global maximum using the composite network model (third column) without realizing that the first function node is more complicated and that its evaluation is more cost-effective. 
Therefore, EIFN first chooses to evaluate in a region close to the initial inferred best solution (black square) and then performs two full evaluations, exploring areas with high uncertainty, such as the boundary at $x=4$, and its inferred best solution upon budget depletion (purple square).

In contrast, p-KGFN takes evaluation costs into account and allocates the budget more efficiently. It first gathers information about the first function node through multiple evaluations (light green triangles) and then evaluates the second node only at the points that it considers most likely to improve the expected solution quality. This behavior yields a more efficient sampling policy which, in turn, results in a more accurate composite function model and inferred best solution (red square).

Similar behaviors emerge when comparing p-KGFN against KGFN with full evaluations, a knowledge-gradient-based acquisition function that also leverages function network but requires full evaluations (see Appendix~\ref{app:pKGFNVsKGFN}).

\section{Maximization of p-KGFN}
\label{sec:optKG}
For simplicity, here we assume that $f_1, f_2, \hdots, f_K$ are scalar-valued.\footnote{Our framework can directly handle multi-output function nodes by employing a multi-output GP model \citep{alvarez2012kernels} for each~$f_k$.} To solve \eqref{eqn:kgdecision}, it suffices to solve 
\begin{equation}
\bmz_k^*\in\argmax_{\bmz_k\in\mathbb{Z}_{n,k}} \alpha_{n,k}(\bmz_k) \label{eqn:individualkg}
\end{equation}
for each node $k$.
Recall that $\mathbb{Z}_{n,k}= \mathbb{Y}_{n,\calJ(k)}\times\mathbb{X}_{\calI(k)}$, where $\mathbb{Y}_{n,\calJ(k)}$ is the discrete set constituted by the outputs from the parent nodes of node $k$ that have been previously generated after $n$ iterations. The discrete nature of $\mathbb{Y}_{n,\calJ(k)}$ makes solving \eqref{eqn:individualkg} challenging.
Additionally, solving~\eqref{eqn:individualkg} presents challenges due to the presence of nested expectations that cannot be computed in closed form, as we explain below. 

To overcome the aforementioned challenges, we propose an approach to compute an approximate solution to \eqref{eqn:individualkg}. Our approach employs sample average approximation (SAA) \citep{kim2015guide,balandat2020botorch}, which substitutes $\alpha_{n,k}(\bmz_k)$ in \eqref{eqn:individualkg} with a Monte Carlo (MC) estimate that is deterministic given a set of finite number of random variables independent of~$\bmz_k$. This is similar to the approach adopted by \citet{astudillo2021bayesian}.  Additionally, to further accelerate computation, we approximate $\nu_{n+1}^*$ by maximizing $\nu_{n+1}$ over a discretization of $\mathbb{X}$, which is similar to the approaches pursued by \citet{scott2011correlated} and \citet{cakmak2020bayesian}. Pseudo-code summarizing the approximate maximization of p-KGFN can be found in Appendix~\ref{appdx:pseudo}.

\subsection{Monte Carlo Estimation of the Outer Expectation}
Recall that the outer expectation in the definition of p-KGFN is over the observation $\bmy_k = \bmf_k(\bmz_k)$ that results from observing node $k$ at $\bmz_k$. To compute this expectation, we use the \textit{reparametrization trick} \citep{kingma2013auto, wilson2018maximizing} to generate samples from the posterior distribution on $\bmf_k(\bmz_k)$. Following 
\citet{balandat2020botorch}, we call these {\it fantasy} samples. They are given by
\begin{equation*}
\hat{y}^{(i)}_k = \mu_{n,k}(\bmz_k) + \sigma_{n,k}(\bmz_k) U^{(i)}
\end{equation*}
where $U^{(i)}, \ i=1, 2, \ldots,I$ are i.i.d. standard normal random variables and $\mu_{n,k}(\cdot)$ and $\sigma_{n,k}(\cdot)$ denote the mean and standard deviation of the GP for node~$k$ at iteration $n$.  

Each fantasy sample, were it actually observed, would generate a new posterior distribution. Let $\nu_{n+1}^{(i)}(\bmx; \bmz_k)$ denote the new posterior mean of $y_K$ at $\bmx$, conditioned on having observed $\hat{y}^{(i)}_k$. An unbiased estimator of $\alpha_{n,k}(z_k)$ is then given by
\begin{equation*}
\left[\frac1I\sum_{i=1}^I \max_{\bmx} \nu_{n+1}^{(i)}(\bmx;\bmz_k) - \nu_n^*\right]/ \,c_k(\bmz_k).
\end{equation*}
\subsection{Monte Carlo Estimation of $\nu_{n+1}$}
\label{sec:postest}
We now discuss computation of $\nu_{n+1}^{(i)}(\bmx; \bmz_k)$.
To explain our approach, we first describe how to generate a sample of the objective function value $y_K(\bmx)$ under a particular posterior distribution. This approach is general, but we focus specifically on the posterior that defines $\nu_{n+1}^{(i)}(\bmx;\bmz_k)$. This is the posterior distribution that conditions on $n$ previous observations and a new observation~$\hat{y}^{(i)}_k$ of $f_k$ at $z_k$. We refer to this distribution as the {\it fantasy-$i$ posterior}.

Fix an index $j$ and let $\bmW^{(j)} = (W_1^{(j)},W_2^{(j)},\hdots,W_K^{(j)})^T \sim \mathcal{N}(0, I_K)$. 
For a generic input~$\bmx$ and the proposed point to sample~$\bmz_k$, define recursively over $\ell=1, 2, \ldots, K$,
\begin{equation}
\begin{split}
   \hat{\bmz}^{(i,j)}_\ell(\bmx;\bmz_k) 
   &:= (\hat{\bmy}^{(i,j)}_{\calJ(\ell)}(\bmx;\bmz_k),\bmx_{I(\ell)})\\
    \hat{y}^{(i,j)}_\ell(\bmx;\bmz_k) 
    &= \mu_{n+1,\ell}^{(i)}(\hat{\bmz}^{(i,j)}_\ell(\bmx;\bmz_k))\\
    &+ \sigma_{n+1,\ell}^{(i)}(\hat{\bmz}^{(i,j)}_\ell(\bmx;\bmz_k))W_\ell^{(j)},
\end{split}
\label{eqn:recursively2}
\end{equation}
where $\mu_{n+1,\ell}^{(i)}(\cdot)$ and $\sigma_{n+1,\ell}^{(i)}(\cdot)$
are the mean and standard deviation of the GP for node $\ell$
under the fantasy-$i$ posterior. We use the notation  $\hat{\bmz}^{(i,j)}_\ell(\cdot; \bmz_k)$ and $\hat{y}^{(i,j)}_\ell(\cdot; \bmz_k)$ to indicate dependence of these quantities on $U^{(i)}$, $\bmW^{(j)}$, and $\bmz_k$.

By construction, $\hat{y}^{(i,j)}_K$ is a sample from the fantasy-$i$ posterior over $y_K(x)$. Thus, we can approximate $\nu_{n+1}^{(i)}(\bmx;\bmz_k)$ by drawing many samples independently from a $K$-dimensional standard normal distribution and averaging the resulting final node value samples obtained via~\eqref{eqn:recursively2}. 
For~$J$ samples, this estimate is given by
$\frac{1}{J}\sum_{j=1}^J \hat{y}^{(i,j)}_K\left(\bmx;\bmz_k\right)$.

\subsection{Putting the Pieces Together}
We can now derive the following MC estimator of the p-KGFN acquisition function:
{ \footnotesize 
\begin{align}
\hat{\alpha}_{n,k}(\bmz_k)=
     \frac{
     \frac{1}{I}\sum_{i=1}^{I}
     \max\limits_{\bmx\in\mathbb{X}} 
     \frac{1}{J}\sum_{j=1}^{J} 
     \hat{y}^{(i,j)}_K\left(\bmx;\bmz_k\right) -\nu^*_{n}
     }{c_k(\bmz_k)}.
\label{eqn:oneshotKG}
\end{align} 
}

We emphasize that the SAA approach relies on fixing the samples $U^{(i)}, \ i=1, 2, \ldots, I$, and  $\bmW^{(j)}, \ j=1, 2, \ldots, K$ that drive the above MC approximation as opposed to generating new samples for each $x$. Thus, the maximization of $\hat{\alpha}_{n,k}(\bmz_k)$ can be seen as a deterministic optimization problem, and its solution as an estimator of the maximizer of $\alpha_{n,k}(\bmz_k)$ (defined in \eqref{eqn:kgfn}). Theorem~\ref{thm:mainthm} shows that this estimator does indeed converge to a maximizer of $\alpha_{n,k}(\bmz_k)$ almost surely as the number of samples increases to infinity. We note that this result requires $J$ to depend on $I$, so we write $J(I)$ to make this dependence explicit. The proof of Theorem~\ref{thm:mainthm} can be found in Appendix~\ref{appdx:proof}.
\begin{restatable}{theorem}{mainthm}
Assume that the prior means $\mu_{0,k'}(\cdot)$ and variances $\sigma_{0,k'}(\cdot)$ are continuous and bounded for all nodes $k'$, that $\mathbb{X}$ and $\mathbb{Z}_{n,k'}$ are compact, and that $\inf_{z\in\mathbb{Z}_{n,k'}} c_{k'}(z)>0$, for all $k'$.
   Consider any node $k$ and write $\hat{\alpha}_{n,k}(z)$ as $\hat{\alpha}_{n,k,I,J(I)}(z)$ to make the dependence on $I$ and $J$ explicit. Then, $\hat{\varphi}_{I,J(I)}:=\max_{z\in\mathbb{Z}_{n,k}}\hat{\alpha}_{n,k,I,J(I)}(z)$ converges to $\varphi^*:=\max_{z\in\mathbb{Z}_{n,k}}\alpha_{n,k}(z)$ almost surely as $I\rightarrow\infty$ where $J$ is a function of $I$ such that $\lim_{I\rightarrow\infty}J(I)=\infty$. Moreover, let $\hat{z}_{I,J(I)}\in\arg\max_{z\in\mathbb{Z}_{n,k}}\hat{\alpha}_{n,k,I,J(I)}(z)$ and $Z^*=\arg\max_{z\in\mathbb{Z}_{n,k}}\alpha_{n,k}(z)$. Then, the distance between $\hat{z}_{I,J(I)}$ and $Z^*$ converges to zero almost surely as $I\rightarrow\infty$.
   \label{thm:mainthm}
\end{restatable}

\subsection{ Discretization of the Inner Problem}\label{subsec:discretization}
To speed up the maximization of $\hat{\alpha}_{n,k}(\bmz_k)$, we discretize the set over which we take the maximum for each fantasy-$i$ posterior in \eqref{eqn:oneshotKG}. I.e., rather than solving the inner maximization problem in~\eqref{eqn:oneshotKG} over the continuous domain $\mathbb{X}$, we instead solve it over a discrete set $\mathbb{A}$.  Similar discretization approaches have been proposed in the literature \citep{scott2011correlated,ungredda2022efficient}.


The set $\mathbb{A}$ can be chosen through several heuristic approaches. Here, we choose $\mathbb{A}$ by taking into account two goals: exploring the promising domain based on the current statistical model, and exploiting the location of the current inferred best solution $\bmx^*_n$. Hence, in each iteration, we form $\mathbb{A}$ using candidates generated by combining the following approaches: First, we draw $N_T$ realizations from the posterior on $y_K$ and include the maximizers of these realizations in $\mathbb{A}$.
Second, we randomly generate $N_L$ local points around $\bmx^*_n$. 
We define a local point $\bmx\in\mathbb{X}$ to be one for which $d(\bmx,\bmx^*_n)\leq r\max_{i=1,2,\hdots,d}(b_i-a_i),$ where $a_i$ and $b_i$ are the lower and upper bounds of $i^{th}$ dimension input, respectively, and $r$ is a positive hyperparameter. Finally, we also include the point $\bmx_n^*$ itself in $\mathbb{A}$. 

An alternative approach to the discretization-based approach described above is to optimize $\hat{\alpha}_{n,k}(\bmz_k)$ in a ``one-shot'' fashion \citep{balandat2020botorch} by introducing a \textit{fantasy variable} $\bmx^{(i)}$ for each index $i$  and then maximizing
\begin{align}
     \frac{\frac{1}{IJ}\sum_{i=1}^{I}
     \sum_{j=1}^{J} 
     \hat{y}^{(i,j)}_K\left(\bmx^{(i)};\bmz_k\right) -\nu^*_{n}}{c_k(z_k)}.
\end{align}  However, this approach results in an optimization problem where the dimension grows linearly in $I$, which in turn results in a substantial increase in computation time. In Appendix~\ref{appdx:OSvsDiscrete}, we compare one-shot optimization and the discretization-based approach we propose below in terms of compute time and solution quality.

\subsection{Further Details}
We maximize $\hat{\alpha}_{n,k}(\bmz_k)$ by enumerating all available (previously evaluated) $\bmy_{\calJ(k)}\in\mathbb{Y}_{n,\calJ(k)}$ and maximize the acquisition function over $x_{\calI(k)}\in\mathbb{X}_{\calI(k)}$ for each $\bmy_{\calJ(k)}$ using gradients with respect to $x_{\calI(k)}$, which we compute using auto-differentiation.  Since $\hat{\alpha}_{n,k}(\bmz_k)$ is deterministic, we use (quasi-)higher-order gradient-based methods, which have been shown to be fast and effective acquisition function optimizers \citep{daulton2020}. We emphasize that this approach is trivially parallelizable: each maximization problem $\max_{\bmz_k}\hat{\alpha}_{n,k}(\bmz_k)$ for each previously evaluated $\bmy_{\calJ(k)}$ \emph{and} for each $k$ can be solved independently and in parallel. Hence, the (wall-)time complexity for optimizing p-KGFN for each previously evaluated $\bmy_{\calJ(k)}$ \emph{and} for each $k$ is the same as solving for a single node $k$ from a single starting point $\bmy_{\calJ(k)}$, given enough parallel compute resources.

\section{Numerical Experiments}
\label{sec:experiments}
We evaluate p-KGFN against several benchmarks, including three algorithms that do not leverage the objective's function network structure: a random sampling baseline (Random), standard versions of expected improvement (EI) and knowledge gradient (KG), and three algorithms that do leverage network structure but require evaluation of the full network: EIFN \cite{astudillo2021bayesian}, a slight modification of EIFN that uses the knowledge gradient instead of EI (KGFN),\footnote{KGFN has not been previously proposed in the literature. Our work is thus the first to describe KG policies for function networks with both partial and full evaluations.} and Thompson sampling for function networks (TSFN). While both p-KGFN and KGFN are one-step lookahead policies, KGFN considers full function network evaluations, whereas p-KGFN obtains one additional observation at one specific node. 
TSFN represents a simple acquisition function leveraging network structures constructed by a series of GP realizations sampled from the nodes' posterior distributions. All algorithms were implemented in BoTorch \citep{balandat2020botorch}. The code to reproduce our experiments is available at \url{https://github.com/frazier-lab/partial_kgfn}.

We assess performance on several function network structures, including single sequential networks, a multi-process network, and a multi-output network. Specifically, we explore two synthetic functions inspired by typical networks in materials design and manufacturing operations, as well as two real-world applications. In our experiments, we consider problems where the upstream nodes must be evaluated before downstream nodes. In this setting, exploiting partial evaluations is usually beneficial when the upstream node is cheaper to evaluate than the downstream node. When the situation is reversed, there are limited gains from using partial evaluations since the expensive upstream nodes must be evaluated before each evaluation of the cheaper downstream nodes. Motivated by real-world scenarios, here we focus on the case where initial nodes are cheaper than later nodes. Moreover, we consider problems without the upstream restriction in Appendix~\ref{app:addexp}.

In all experiments, each algorithm begins with $2d+1$ points chosen at random over the input space $\mathbb{X}\subseteq\mathbb{R}^d$.
Each point $\bmx\in\mathbb{X}$ is fully evaluated across the entire network (i.e., we observe $y_k(\bmx)$ for $k=1,\hdots, K$). Then, at each iteration, each algorithm sequentially selects a point to evaluate. 
All six baselines choose a point $\bmx\in\mathbb{X}$ and evaluate the entire network. In contrast, p-KGFN can take advantage of partial evaluations by selecting both a node $k$ and its input $\bmz_k\in\mathbb{Z}_{n,k}$ to evaluate at each iteration. 
All experiments discussed in this section are noise-free. We conduct an additional experiment with noisy observations in Appendix~\ref{appdx:addtl_exp_noisy} to show the robustness of p-KGFN.

We evaluate the performance of each algorithm by reporting at each iteration the ground truth value of $y_K(\bmx^*_n)$, where $\bmx^*_n\in \argmax_{x\in\mathbb{X}}\nu_n(x)$. To highlight the benefits of partial evaluations, we utilize a posterior distribution for the final node value $y_K$ obtained from a statistical model that incorporates the network structure discussed in Section~\ref{sec:statmodel} to compute the metric for all algorithms. Note that this favors algorithms such as EI, KG, and Random, which make decisions without leveraging the network structure (results when using a model not incorporating network structure are presented in Appendix~\ref{app:alter}). Averaging over 30 replications, we report the mean of this metric with the error bars showing two standard errors.

\subsection{Synthetic Test Problems}
\label{subsec:syn}
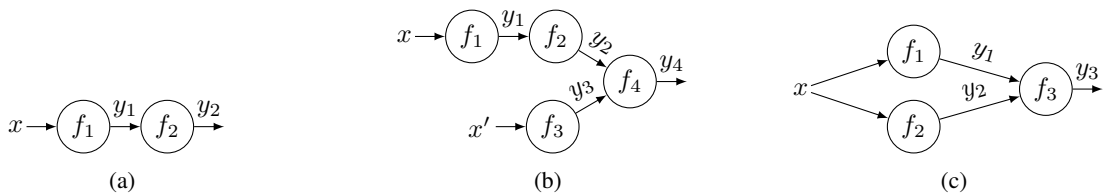
\begin{figure*}[!t]
\centering
\begin{subfigure}{0.33\textwidth}
\centering
\begin{tikzpicture}[
init/.style={
  draw,
  circle,
  inner sep=0.7pt,
  minimum size=0.7cm
},
init2/.style={
  circle,
  inner sep=0.2pt,
  minimum size=0.2cm
},
]
\begin{scope}[start chain=1,node distance=4mm]
\node[on chain=1, init2]
(x1) {$\bmx$};
\node[on chain=1, init] 
  (f1) {$f_1$};
\node[on chain=1,init]
 (f2) {$f_2$};
\node[on chain=1,init2] (f3) 
  {};
\end{scope}

\draw[-latex] (f1) -- (f2)node[pos=0.5,sloped,above] {$y_1$};
\draw[-latex] (f2) -- (f3)node[pos=0.5,sloped,above] {$y_2$};
\draw[-latex] (x1) -- (f1);
\end{tikzpicture}
\caption{}
\label{fig:eggandfree}
\end{subfigure}%
\begin{subfigure}{0.33\textwidth}
\centering
\begin{tikzpicture}[
init/.style={
  draw,
  circle,
  inner sep=0.7pt,
  minimum size=0.7cm
},
init2/.style={
  circle,
  inner sep=0.2pt,
  minimum size=0.2cm
},
]
\begin{scope}[start chain=1,node distance=4mm]
\node[on chain=1, init2]
(x1) {$\bmx$};
\node[on chain=1, init] 
  (f1) {$f_1$};
  \node[on chain=1, init]
  (f2) {$f_2$};
\end{scope}
\begin{scope}[start chain=2,node distance=4mm]
\node[on chain=2,init] (f4) at (3,-6mm) {$f_4$};
\node[on chain=2,init2] (f5) 
  {};
\end{scope}
\begin{scope}[start chain=3,node distance=4mm]
\node[on chain=3, init2] at (1,-12mm)
(x2) {$\bmx'$};
\node[on chain=3,init]
 (f3) {$f_3$};
\end{scope}

\draw[-latex] (f1) -- (f2)node[pos=0.5,sloped,above] {$y_1$};
\draw[-latex] (f2) -- (f4)node[pos=0.5,sloped,above] {$y_2$};
\draw[-latex] (x1) -- (f1);
\draw[-latex] (x2) -- (f3);
\draw[-latex] (f3) -- (f4)node[pos=0.5,sloped,above] {$y_3$};
\draw[-latex] (f4) -- (f5)node[pos=0.5,sloped,above] {$y_4$};
\end{tikzpicture}
\caption{}
\label{fig:manu}
\end{subfigure}
\begin{subfigure}{0.3\textwidth}
\centering
\begin{tikzpicture}[
init/.style={
  draw,
  circle,
  inner sep=0.7pt,
  minimum size=0.7cm
},
init2/.style={
  circle,
  inner sep=0.2pt,
  minimum size=0.2cm
},
]
\begin{scope}[start chain=1,node distance=4mm]

\node[on chain=1, init] 
  (f1) at(1,-1mm){$f_1$};
\end{scope}
\begin{scope}[start chain=2,node distance=4mm]
\node[on chain=2, init2]
(x1)at(-0.5,-6mm){$\bmx$};
\node[on chain=2,init] (f3) at (2,-6mm) {$f_3$};
\node[on chain=2,init2] (f4) 
  {};
\end{scope}
\begin{scope}[start chain=3,node distance=4mm]
\node[on chain=3,init]
 (f2) at(1,-11mm){$f_2$};
\end{scope}

\draw[-latex] (f1) -- (f3)node[pos=0.5,sloped,above] {$y_1$};
\draw[-latex] (f2) -- (f3)node[pos=0.5,sloped,above] {$y_2$};
\draw[-latex] (x1) -- (f1);
\draw[-latex] (x1) -- (f2);
\draw[-latex] (f3) -- (f4)node[pos=0.5,sloped,above] {$y_3$};
\end{tikzpicture}
\caption{}
\label{fig:pharma}
\end{subfigure}
\caption{Function network structures in the numerical experiments: (a) Ackley and FreeSolv, (b) Manu-GP, and (c) Pharma.}
\label{fig:problemnetwork}
\end{figure*}
\begin{figure*}[!h]
    \centering
    \includegraphics[width=0.95\textwidth]{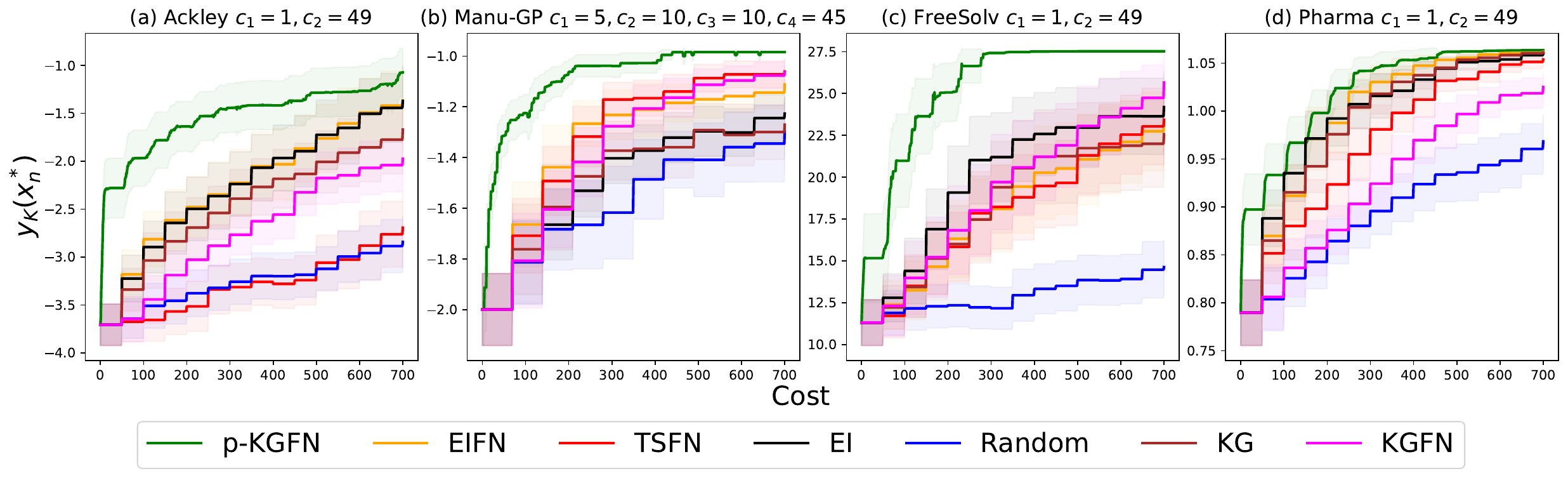}
    \caption{Optimization performance comparing between our proposed p-KGFN and benchmarks including EIFN, KGFN, TSFN, EI, KG and Random on four experiments: (a) Ackley, (b) Manu-GP, (c) FreeSolv, and (d) Pharma.}
    \label{fig:performance}
\end{figure*}
\begin{figure*}[!h]
    \centering
    \includegraphics[width=0.7\textwidth]{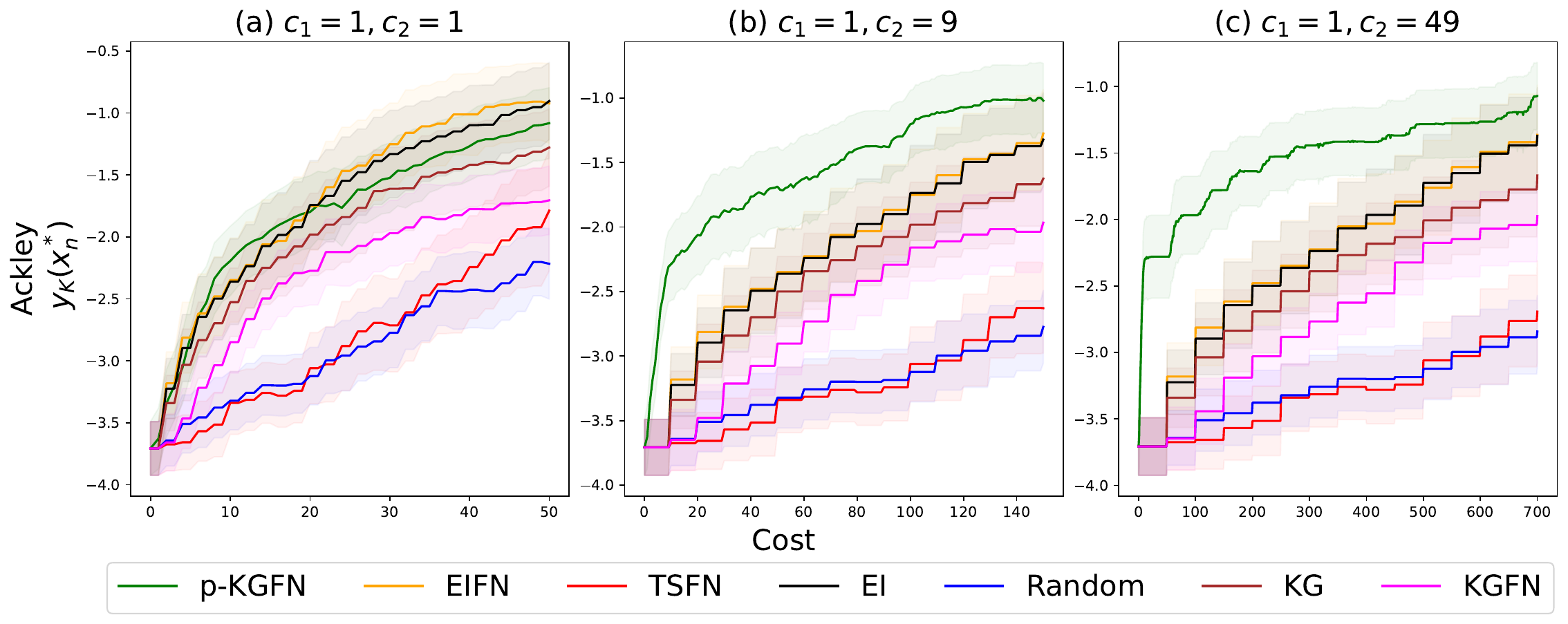}
    \caption{Cost sensitivity analysis for Ackley problem with different costs (a) $c_1 = 1, c_2 = 1$; (b) $c_1=1, c_2=9$; and (c) $c_1 = 1, c_2 = 49$.}
    \label{fig:AckScost}
\end{figure*}

\paragraph{Ackley6D (Ackley)} This problem is structured as a two-stage function network (Figure \ref{fig:eggandfree}), where the first stage takes a 6-dimensional input and the second stage takes as an input the output of the first stage. The node function $f_1$ is the negated Ackley function \citep{jamil2013literature}, and $f_2$ is given by $f_2(y_1) = -y_1\times\sin(5y_1/6 \pi)$. This network structure is commonly found in materials design, sequential processes, and multi-fidelity settings. In many applications, the early stages in a path through the function network are cheaper to evaluate than subsequent stages. For example, the first node can be an approximation or partial evaluation of a subsequent node. We therefore assume the costs are given by $c_1 = 1$ and $c_2=49$.

\paragraph{Manufacturing Network (Manu-GP)} Motivated by the manufacturing application discussed in Section~\ref{sec:intro}, we build a two-process network where the outputs of the two processes are combined at a final node (Figure \ref{fig:manu}). The first process has two sequential nodes and its initial input is 2-dimensional. The second process has one node which takes a different 2-dimensional input. This network structure is typical in scenarios where individual components are produced through independent processes and combined to create a final product (e.g., chemical synthesis processes). In this experiment, we employ a sample path drawn from a GP prior for each function node. This is intended to emulate the variations in the characteristics of intermediate/final products, with respect to different design parameters. Since different processes typically have different levels of complexity, resulting in heterogeneous costs, we assume that $c_1 = 5$ and $c_2=10$ in the first process, and $c_3=10$ in the second process. As the final stage usually involves both component assembly and product quality assessment, we assume the final stage incurs a relatively high cost, $c_4=45$.

\subsection{Real-World Applications}
\label{subsec:experiments:realapp}

\paragraph{Molecular Design (FreeSolv)} We consider the FreeSolv dataset \citep{mobley2014freesolv}, which consists of calculated and experimental hydration-free energies of 642 small molecules. A continuous representation is derived from the SMILES representation of each molecule through a variational autoencoder model \citep{gomez2018automatic}. We apply principal component analysis (PCA) to reduce the dimension of these representations to three. In the context of materials design, our objective is to minimize the experimental free energy. We formulate this problem as a two-stage function network (Figure~\ref{fig:eggandfree}). The first node takes the 3-dimensional representation of molecules as input and outputs the negative calculated free energy. The second node takes this output as input and returns the negative experimental free energy, which is our target for maximization. We fit GP surrogate models for both nodes based on the entire available dataset, which allows us to consider a continuous optimization domain. As in the Ackley problem, we assume $c_1=1$ and $c_2=49$.

\paragraph{Pharmaceutical Product Development (Pharma)} 

An orally disintegrating tablet (ODT) is a drug dosage form designed to dissolve on the tongue.
To ensure the production of high-quality ODTs, one must consider two crucial properties: disintegration time $(f_1)$ and tensile strength ($f_2$). 
We employ surrogate models proposed in \citet{sano2020application} for these two target properties as functions of four input variables in the production process.  
In the same study, a simple deterministic score function $(f_3)$, which combines these two targets, is proposed to measure the ODT quality (Figure~\ref{fig:pharma}). 
Our goal is to find the four input variables that maximize the score produced by node~$f_3$. 
While both target properties are equally significant in determining the quality of the ODT, their measurements typically involve different levels of complexity. To reflect this, 
costs are set at $c_1=1$ for $f_1$ and $c_2=49$ for $f_2$. The $f_3$ score function, being already known and cost-free to evaluate, is not included as a node in the p-KGFN optimization process. 

\subsection{Results and Discussion}\label{sec:discussion}

Figure~\ref{fig:performance} shows the performance of p-KGFN compared to various baselines across different experiments. While the baselines require full evaluations of the function networks, resulting in regular-spaced steps in their performance curves, p-KGFN operates without this limitation. We note that p-KGFN outperforms all baselines across all experiments.

In the FreeSolv and Ackley experiments, similar to the toy example from  Section~\ref{subsec:advantage_partial}, p-KGFN demonstrates strategic budget allocation by learning the first function node thoroughly before deciding to explore the second node, and only in regions with high potential values.  
To assess the impact of evaluation costs on performance, we conduct a sensitivity analysis where we vary the cost of evaluating the second node in two-node network problems, as shown in Figure~\ref{fig:AckScost} and detailed in Appendix~\ref{appdx:sensitivity}. We observe that when the costs of the two nodes are equal, p-KGFN typically performs full evaluations, achieving performance levels comparable to baseline algorithms. The advantages of partial evaluations become more pronounced as the evaluation cost of the second node increases.

Comparing results across different test problems, we note that p-KGFN achieves more significant performance gains in scenarios where downstream nodes in the function network show strong correlations with their parent nodes (as seen in the Ackley and FreeSolv experiments) and involve higher evaluation costs. Such dynamics are common in real-world settings, where an upstream node might simulate a scenario that is physically tested in a downstream node. To explore the robustness of p-KGFN, we conduct additional experiments where upstream nodes are more challenging to model than downstream nodes, detailed in Appendix~\ref{appdx:addtl_exp_hard_downstream}. In such experiments, p-KGFN's performance remains on par with other benchmarks, confirming its effectiveness across diverse problem settings.

Overall, the networks evaluated in the numerical experiments showcase p-KGFN's capability to effectively manage a diverse range of network structures. These include sequential networks, multi-process networks, and multi-output networks with clearly defined objectives.

\section{Conclusion}
\label{sec:conclusion}

In this work, we considered Bayesian optimization of objectives represented by a network of functions, where individual nodes in the network can be evaluated independently. We proposed a new acquisition function for this class of problems, p-KGFN, that leverages the objective's function network structure along with the ability to evaluate individual nodes to improve sampling efficiency. Our numerical experiments on both synthetic and real-world problems demonstrate that our approach can significantly reduce evaluation costs and provide higher-quality solutions.

While our method offers substantial benefits through partial evaluations, it is also subject to some limitations. Specifically, optimizing the p-KGFN acquisition function requires considerable computational resources as it considers all nodes and available outputs at each iteration, which could be challenging for large networks (the average runtime for optimizing each acquisition function is reported in Appendix~\ref{supp:wallclock}).
Nonetheless, in many real-world scenarios where evaluation costs are high, the savings achieved through an improved query strategy significantly outweigh the additional computational time. Additionally, it may be possible to extend and integrate the stock reduction technique from \citet{kusakawa2022bayesian}, designed for cascade-type networks, to a broader range of network structures to reduce the number of optimization problems considered in p-KGFN. We leave this as a direction for future work. Finally, our algorithm, like other knowledge-gradient-based algorithms, looks only a single step ahead. An interesting research direction would be to explore multi-step lookahead acquisition functions \cite{jiang2020efficient} for function networks, though this would further increase the computational cost. 

\section*{Acknowledgements}
We would like to thank the anonymous reviewers for their comments. We also thank Eytan Bakshy for his valuable feedback. PB would like to thank DPST scholarship project granted by IPST, Ministry of Education, Thailand for providing financial support. PF and JW were supported by AFOSR FA9550-19-1-0283 and FA9550-20-1-0351.

\section*{Impact Statement}
This paper presents work whose goal is to advance the field of Machine Learning. There are many potential societal consequences of our work, none which we feel must be specifically highlighted here.

\bibliography{ref}
\bibliographystyle{icml2024}
\newpage
\appendix
\onecolumn

\section{Pseudo-Code for the p-KGFN Algorithm}\label{appdx:pseudo}

We present the pseudo-code for implementing Bayesian optimization with the p-KGFN acquisition function, supplementing the descriptions in Section~\ref{sec:optKG}. Algorithm~\ref{algo:bo_loop} outlines the BO loop employing the p-KGFN algorithm. Algorithm~\ref{algo:KGFN} describes the computation of the MC estimate of the acquisition value. Algorithm~\ref{algo:SAA} describes how we estimate the posterior mean of the final function node via MC simulation, which is necessary for Algroithm~\ref{algo:KGFN}. We defer the discussion of acquisition optimization to Appendix~\ref{appdx:opt}, where we compare our implementation, optimization-via-discretization, against a commonly used approach, one-shot optimization.

\begin{algorithm}[]
    \caption{Bayesian Optimization using p-KGFN} \label{algo:bo_loop}
    {\bfseries Input:}\\
    $c_k(\cdot)$, the evaluation cost function for node $k$, $k=1,\ldots, K$\\
    $B$, the total evaluation budget\\
    $\mu_{0, k}(\cdot)$ and $\sigma_{0, k}(\cdot)$, the mean and standard deviation of the GP for node $k$, $k=1,\ldots, K$ (fitted using initial observations)\\
    {\bfseries Output:} the point with the largest posterior mean at the final function node
    \begin{algorithmic}[1]
    \STATE $n\leftarrow 0$ 
    \STATE $b\leftarrow 0$
    \WHILE{$b < B$}
        \STATE $n\leftarrow n + 1$
        \FOR{$k=1,\ldots, K$}
            \STATE identify the set of combinations of previously evaluated $y_{\calJ(k)}$, $\mathbb{Y}_{n,\calJ(k)}$
            \IF{$\mathbb{Y}_{n,\calJ(k)}=\emptyset$}
            \STATE $\hat{\alpha}_{n, k}^{*}\leftarrow -1$
            \ELSE
            \STATE $\hat{\alpha}_{n, k}^{*}\leftarrow \max_{z\in \mathbb{Z}_{n,k}}\hat{\alpha}_{n, k}(z)$ where $\hat{\alpha}_{n,k}(\cdot)$ is computed using Algorithm~\ref{algo:KGFN}
            \STATE $z_k^* \leftarrow \argmax_{ z\in \mathbb{Z}_{n,k}} \hat{\alpha}_{n,k}(z)$
            \IF{$c_k(z_k^*) > B - b$}
            \STATE{$\hat{\alpha}_{n, k}^{*}\leftarrow -1$}
            \ENDIF
            \ENDIF
        \ENDFOR
        \IF{$\max_k\hat{\alpha}_{n, k}^{*} = -1$}
        \STATE \textbf{break}
        \ELSE
        \STATE $k^* \leftarrow \argmax_{k\in \{1,\ldots, K\}}\hat{\alpha}_{n, k}^{*}$
        \STATE obtain $y_{k^*} = f_{k^*}(z^*_{k^*})$
        \STATE update the GP model for node $k^*$ with the additional observation $(z^*_{k^*}, y_{k^*})$
        \STATE $b\leftarrow b + c_{k^*}(z_{k^*}^*)$
        \ENDIF
    \ENDWHILE
    \end{algorithmic}
    \textbf{return} $\argmax_{x\in\mathbb{X}}\hat{\nu}_n(x)$, an estimate of $\nu_{n}(\bmx)$ given in Algorithm~\ref{algo:SAA} using a gradient-based method
\end{algorithm}
\vfill
\newpage

\begin{algorithm}
    \caption{MC Estimate of $\alpha_{n,k}(\bmz_k)$}\label{algo:KGFN}
    {\bfseries Input:} \\
     $k$, the node to be evaluated \\
    $\bmz_k$, the input for node $k$ \\
    $c_k(\cdot)$, the evaluation cost function for node $k$ \\
    $I$, the number of fantasy observations to create\\
    $J$, the number of MC samples for estimating the posterior mean of the final function node\\
    $\mu_{n, k}(\cdot)$ and $\sigma_{n, k}(\cdot)$,  the mean and standard deviation of the GP for node $k$, $k=1,\ldots, K$ \\
    {\bfseries Output:} $\hat{\alpha}_{n,k}(\bmz_k)$, the estimated acquisition value
    \begin{algorithmic}[1]
        \STATE solve $\max_{\bmx\in\mathbb{X}}\hat{\nu}_{n}(\bmx)$, an estimate of $\nu_{n}(\bmx)$ given in Algorithm~\ref{algo:SAA} using a gradient-based method and obtain $\hat{\nu}_{n}^*$
        \STATE generate $I$ independent standard normal random variables $U^{(i)}$, $i=1,\ldots, I$
        \FOR{$i=1, \ldots, I$}  
            \STATE $\hat{\bmy}_k^{(i)} \leftarrow \mu_{n, k}(\bmz_k) + \sigma_{n,k}(\bmz_k){U}^{(i)}$
            \STATE update the posterior of GP for node $k$ with the additional observation $\left(\bmz_k,\hat{\bmy}_k^{(i)}\right)$

            \STATE solve $\max_{\bmx\in\mathbb{X}}\hat{\nu}^{(i)}_{n+1}(\bmx)$ to obtain $\hat{\nu}_{n+1}^{(i)*}$ for fantasy-$i$ model 
             using Algorithm~\ref{algo:SAA}

        \ENDFOR
        \STATE $\hat{\alpha}_{n,k}(\bmz_k) \leftarrow \frac{1}{c_k(\bmz_k)} \left(\frac{1}{I}\sum_{i=1}^{I}\hat{\nu}^{(i)*}_{n+1} - \hat{\nu}_n^*\right)$ 
        \STATE \textbf{return} $\hat{\alpha}_{n,k}(\bmz_k)$
    \end{algorithmic}
\end{algorithm}
Note that in Line 6 of Algorithm~\ref{algo:KGFN}, the samples $W^{(j)}$ used in Algorithm~\ref{algo:SAA} are shared across the MC approximation for all fantasy-$i$ models.

\begin{algorithm}[h]
    \caption{Posterior Mean Estimate of the Final Function Node via MC Simulation}\label{algo:SAA}
    {\bfseries Input:} \\
    $\bmx\in\mathbb{X}$, a design point of the function network \\
    $J$, the number of MC samples \\
    $\mu_{k}(\cdot)$ and $\sigma_{k}(\cdot)$, the mean and standard deviation of the GP for node $k$, for $k=1,\ldots, K$\\
    {\bfseries Output:} $\hat{\nu}(x)$, the estimated posterior mean
    \begin{algorithmic}[1]
    \STATE generate $J$ independent samples $\bmW^{(j)} = \left(W_1^{(j)},W_2^{(j)},\hdots,W_K^{(j)}\right)^T $ for $j=1,\ldots, J$ from $\mathcal{N}(0, I_K)$; 
    \FOR{$j=1, \ldots, J$}
        \FOR{$k=1,\ldots,K$}
            \STATE define $\hat{\bmz}^{(j)}_k(\bmx) := (\hat{\bmy}^{(j)}_{\calJ(k)}(\bmx),\bmx_{\calI(k)})$
            \STATE $\hat{y}^{(j)}_k(\bmx) \leftarrow \mu_{k}(\hat{\bmz}^{(j)}_k(\bmx))+ \sigma_{k}(\hat{\bmz}^{(j)}_k(\bmx))W_k^{(j)}$
        \ENDFOR
    \ENDFOR
    \STATE $\hat{\nu}(\bmx)\leftarrow\frac{1}{J}\sum_{j=1}^{J}\hat{\bmy}^{(j)}_K(\bmx)$
    \STATE \textbf{return} $\hat{\nu}(\bmx)$
    \end{algorithmic}
\end{algorithm}

\section{Proof of Theorem~\ref{thm:mainthm}}
\label{appdx:proof}
In this section, we prove the following theorem.
\mainthm*
Our proof is mainly based on Lemma A1 and Theorem A1 in \citetAP{rubinstein1993discrete}. We will assume throughout this section the assumptions in the statement of Theorem~\ref{thm:mainthm} --- that the prior means $\mu_{0,k'}(\cdot)$ and variances $\sigma_{0,k'}(\cdot)$ are continuous and bounded for all $k'$, that $\mathbb{X}$ and $\mathbb{Z}_{n,k'}$ are compact, and that the cost $c_{k'}(z)$ is bounded below away from $0$.

To support this proof, we first introduce the notation used in the statements and proofs of the lemmas below. Observe that the posterior distribution of $f_k$ given the data up to time $n$ and an additional observation of $f_k$ at $z$ is a Gaussian process whose posterior mean and posterior variance can be written using the standard Gaussian process regression formulas
\citepAP{williams2006gaussian}:
\begin{equation*}
\mu_{n+1,k}(z')
= \mu_{n,k}(z')
+ \Sigma_{n,k}(z',z)(\Sigma_{n,k}(z,z)+\lambda)^{-1}(f_k(z)-y_{n,k}(z)),
\end{equation*}
and
\begin{align*}
\sigma_{n+1,k}(z')
&= \Sigma_{n,k}(z',z') - 
  \Sigma_{n,k}(z',z)(\Sigma_{n,k}(z,z)+\lambda)^{-1}\Sigma_{n,k}(z',z)\\
&= 
\Sigma_{n,k}(z',z') - \Sigma_{n,k}(z',z)^2 / (\Sigma_{n,k}(z,z)+\lambda).
\end{align*}

Let $U = (f_k(z)-y_{n,k}(z))/\sqrt{\Sigma_{n,k}(z,z)+\lambda}$, which is a standard normal random variable. Then, we can rewrite $\mu_{n+1, k}$ as
\begin{equation*}
\mu_{n+1,k}(z')
= \mu_{n,k}(z')
+ \Sigma_{n,k}(z',z)U / \sqrt{\Sigma_{n,k}(z,z)+\lambda}.
\end{equation*}

Define $\mu_{n+1,\ell}=\mu_{n,\ell}$ and $\sigma_{n+1,\ell}=\sigma_{n,\ell}$ for $\ell \ne k$.

Now, we can create a recursive expression for a sample from the posterior over $y_K(x)$ given these $n+1$ evaluations in terms of the posterior means and variances and a sequence of standard normal random variables $W = (W_\ell : \ell=1,\ldots,K)$. Specifically, define recursively for each $\ell=1,\ldots, K$,
\begin{equation*}
\begin{split}
   {\bmz}_\ell(\bmx;\bmz) 
   &:= ({\bmy}_{J(\ell)}(\bmx;\bmz),\bmx_{I(\ell)})\\
    {y}_\ell(\bmx;\bmz) 
    &= \mu_{n+1,\ell}({\bmz}_\ell(\bmx;\bmz))
    + \sigma_{n+1,\ell}({\bmz}_\ell(\bmx;\bmz))W_\ell.
\end{split}
\end{equation*}
Making a slight abuse, we make the dependence of $y_K(x)$ on, $z$, $U$, and $W$ explicit through the notation $y_K(U, W, x, z)$.

Now, consider a specific node $k$. We drop the subscript $n,k$ for convenience. We define $h(U,W,x,z)=\frac{y_K(U,W,x,z)}{c(z)}$. In addition, define 
$$\alpha(z) = \mathbb{E}[\max_{x\in\mathbb{X}}\mathbb{E}[h(U,W,x,z)|U]]$$
and
$$\hat{\alpha}_{I,J(I)}(z)=\frac{1}{I}\sum_{i=1}^I\max_{x\in\mathbb{X}}\frac{1}{J(I)}\sum_{j=1}^{J(I)}h(U_i,W_j,x,z),$$
where $J(I)$ will be defined later.

Before proving the theorem~\ref{thm:mainthm}, we prove several auxiliary lemmas.

\begin{lemma}
\label{lemma:conth}
    For almost every $(u,w)$, $h(u,w,\cdot,\cdot)$ is continuous.
\end{lemma}
\begin{proof}[Proof of Lemma~\ref{lemma:conth}]
Since the prior means $\mu_{0,k'}(\cdot)$ variances $\sigma_{0,k'}(\cdot)$ are continuous, a simple argument shows that the posterior means $\mu_{n+1,k'}(\cdot)$ and variances $\sigma_{n+1,k'}(\cdot)$ are continuous too. Thus, $h(u,w,x,z)$ is a composition of continuous functions and so is also continuous.
\end{proof}

\begin{lemma}
\label{lemma:yK_bound}
There exist finite non-negative constants $c_0$, $c_1$ and $c_2$ not depending on $U$, $W$, $x$ or $z$ such that 
$|y_K(U,W,x,z)| \le c_0 + c_1 |W_K| + c_2 |U|$.
\end{lemma}
\begin{proof}[Proof of Lemma~\ref{lemma:yK_bound}]
Consider two cases. In the first case, suppose that the node we are measuring, $k$, precedes the final node, $k<K$. In this case, the posterior mean and variance of node $K$ do not change and are equal to their values under the prior.
Thus, $y_K(U,W,x,z)$ is equal to the prior mean function $\mu_{n,K}(\cdot)$ evaluated at a random input 
$z_K(x;z)$ plus the noise term $\sigma_{n,K}(z_K(x;z))W_K$.
This is bounded above by $\sup_{z'\in\mathbb{Z}_{n,K}}|\mu_{n,K}(z')| + \sup_{z'\in\mathbb{Z}_{n,K}}|\sigma_{n,K}(z')||W_K|$.
Note that $c_0 = \sup_{z'\in\mathbb{Z}_{n,K}}|\mu_{n,K}(z')|$ and $c_1 = \sup_{z'\in\mathbb{Z}_{n,K}}|\sigma_{n,K}(z')|$ are finite.

In the second case, we measure the final node, $k=K$.
Then, 
\begin{equation*}
\mu_{n+1,K}(z')
= \mu_{n,K}(z')
+ \tilde{\sigma}_{n,K}(z',z) U
\end{equation*}
where $\tilde{\sigma}_{n,K}(z',z) = \Sigma_{n,K}(z',z) / \sqrt{\Sigma_{n,K}(z,z)+\lambda}$.

Thus,
\begin{equation*}
y_K(U,W,x,z)  
= \mu_{n,K}(z_K(x;z))
+ \tilde{\sigma}_{n,K}(z_K(x;z),z) U
+ \sigma_{n+1,K}(z') W_K
\end{equation*}
and we have that 
\begin{equation}
|y_K(U,W,x,z)| \le 
|\mu_{n,K}(z_K(x;z))|
+ |\tilde{\sigma}_{n,K}(z_K(x;z),z)| |U|
+ |\sigma_{n+1,K}(z')| |W_K|.
\label{eq:lemma_rhs}
\end{equation}

Observe that $\tilde{\sigma}_{n,K}$, $\mu_{n,K}(z')$ and $\sigma_{n+1,K}$ are continuous and thus bounded over their compact domains.
Thus, the right-hand side of \eqref{eq:lemma_rhs} can be written in the form claimed in the statement of this lemma.
\end{proof}

\begin{lemma}
\label{lemma:dominated}
    The family $\{|h(u,w,x,z)|:x\in\mathbb{X},z\in\mathbb{Z}\}$ is dominated by an integrable function (with respect to the probability distribution over $U$ and $W$), i.e. $|h(u,w,x,z)|\leq a(u,w)$, where $\mathbb{E}[a(U,W)]<\infty$. Moreover, $b_r(w):=\max_{\{u:|u|\leq r\}}a(u,w)$ is also integrable for each $r<\infty$.
\end{lemma}
\begin{proof}[Proof of Lemma~\ref{lemma:dominated}]
We will show the statement of the lemma for the case where $c(x)=1$. The general case follows from this because we can replace $a(u,w)$ by $a(u,w) / \min_{z\in\mathbb{Z}} c(z)$, where the denominator is strictly positive by our assumptions.

Setting $a(U,W) = c_0 + c_1 |W_K| + c_2 |U|$ using the constants $c_0$, $c_1$ and $c_2$ from Lemma~\ref{lemma:yK_bound} shows that $|h(u,w,x,z)|\leq a(u,w)$, where $\mathbb{E}[a(U,W)]<\infty$.

Moreover, letting $b_r(W):=\max_{\{u:|u|\leq r\}}a(u,w) = c_0 + c_1 |W_K| + c_2 |r|$, we have that 
$b_r(W)$ is integrable.

\end{proof}

\begin{lemma}
    \label{lemma:pointwise}
    For each $z\in\mathbb{Z}$, $\hat{\alpha}_{I,J(I)}(z)\rightarrow\alpha(z)$ almost surely.
\end{lemma}
\begin{proof}[Proof of Lemma~\ref{lemma:pointwise}]
    Fix $z\in\mathbb{Z}$. Let $$\hat{\beta}_I(U;z)=\max_{x\in\mathbb{X}}\frac{1}{J(I)}\sum_{j=1}^{J(I)}h(U,W_j,x,z)$$
    and
    $$\beta(U;z)=\max_{x\in\mathbb{X}}\mathbb{E}[h(U,W,x,z)|U].$$
    Observe that the mapping $(x,u)\rightarrow h(u,w,x,z)$ satisfies the assumptions of Lemma A1 in \citetAP{rubinstein1993discrete}, when $u$ is restricted to the set $\{u:|u|\leq r\}$, i.e. 
    \begin{itemize}
        \item $\mathbb{X}\times \{u:|u|\leq r\}$ is compact.
        \item The mapping $(x, u)\rightarrow h(u,w,x,z)$ is continuous for almost every $w$.
        \item $h(u,w,x,z)\leq b(w)=\max_{u:|u|\leq r}a(u,w)$ for all $x\in\mathbb{X}$ and $u\in[-r,r]$, and $\mathbb{E}|b(W)|<\infty$ by Lemma~\ref{lemma:dominated}.
    \end{itemize} Thus we have by the Lemma A1, w.p. 1,
    $\frac{1}{J(I)}\sum_{j=1}^{J(I)}h(U,W_j,x,z)$ converges to $\mathbb{E}[h(U,W,x,z)|U]$ uniformly in $x,U$ over $\mathbb{X}\times [-r,r]$.

    Choose any $\epsilon>0$ and let $I_1$ be large enough on the given sample path that 
    $$\left|\frac{1}{J(I)}\sum_{j=1}^{J(I)}h(U,W_j,x,z)-\mathbb{E}[h(U,W,x,z)|U]\right|<\epsilon,$$
    for all $I\geq I_1$, $U\in[-r,r]$ and $x\in\mathbb{X}$. Then,
    $$|\hat{\beta}_{I}(U;z)-\beta(U;z)|<\epsilon,$$
    for all $I\geq I_1$ and $U\in[-r,r]$.\\
    Let 
        \[ U^{(r)} = \begin{cases} 
          U & \text{if}\hspace{0.5cm} U\in[-r,r]\\
          r & \text{if}\hspace{0.5cm} U>r \\
          -r & \text{if}\hspace{0.5cm} U<-r.
       \end{cases}
    \]
    and similarly for $U_i^{(r)}$.\\
    Let $\hat{\alpha}^{(r)}_I(z)=\frac{1}{I}\sum_{i=1}^I\hat{\beta}_I(U_i^{(r)};z)$ and $\alpha^{(r)}(z)=\mathbb{E}[\beta(U^{(r)};z)]$. We have
    \begin{align}
        \left|\hat{\alpha}_I(z)-\alpha(z)\right|&\leq \left|\hat{\alpha}_I(z)-\hat{\alpha}^{(r)}_I(z)\right|+\left|\hat{\alpha}^{(r)}_I(z)-\alpha^{(r)}(z)\right|+\left|\alpha^{(r)}(z)-\alpha(z)\right|
    \end{align}
    For $I\geq I_1$, the second term is bounded above by $\epsilon$. Let's focus on the first and third terms. We have that the third term is bounded above by
    \begin{align}
        \left|\alpha^{(r)}(z)-\alpha(z)\right|&\leq \mathbb{E}\left[\mathbf{1}_{\{|U|>r\}}|\beta(U)-\beta(U^{(r)})|\right]\notag\\
        &\leq\mathbb{E}\left[\mathbf{1}_{\{|U|>r\}}(|\beta(U)|+|\beta(U^{(r)})|)\right]\notag\\
        &\leq\mathbb{E}\left[\mathbf{1}_{\{|U|>r\}}(a(U,W)+a(U^{(r)},W))\right]\label{eqn:proof_new_0},
    \end{align}
    where the last inequality comes from 
    \begin{align}
        |\beta(U)|&=|\max_{x\in\mathbb{X}}\mathbb{E}[h(U,W,x,z)|U]|\notag\\
        &\leq \max_{x\in\mathbb{X}}\mathbb{E}[a(U,W)|U]=\mathbb{E}[a(U,W)|U].
    \end{align}
    Consider the right hand side of \eqref{eqn:proof_new_0}, we have that
    \begin{align}
        \mathbb{E}[\mathbf{1}_{\{|U|>r\}}(a(U,W)&+a(U^{(r)},W))] \notag\\
        &=\mathbb{E}\left[\mathbf{1}_{\{|U|>r\}}(2c_0+2c_1|W_K|+c_2|U|+|U^{(r)}|)\right]\notag\\
        &=2c_0\mathbb{P}(|U|>r)+2c_1\mathbb{P}(|U|>r)\mathbb{E}[|W_K|]+c_2\mathbb{E}[\mathbf{1}_{\{|U|>r\}}|U|]+\mathbb{E}[\mathbf{1}_{\{|U|>r\}}|U^{(r)}|]\notag\\
        &=2c_0\mathbb{P}(|U|>r)+2c_1\mathbb{P}(|U|>r)\mathbb{E}[|W_K|]+c_2\mathbb{E}[\mathbf{1}_{\{|U|>r\}}|U|]+\mathbb{E}[\mathbf{1}_{\{|U|>r\}}|r|]\notag\\
        &=2c_0\mathbb{P}(|U|>r)+2c_1\mathbb{P}(|U|>r)\mathbb{E}[|W_K|]+c_2\mathbb{E}[\mathbf{1}_{\{|U|>r\}}|U|]+r\mathbb{P}(|U|>r).\label{eqn:rhs_r}
    \end{align}
    Since $U$ is a standard normal random variable, $\mathbb{P}(|U|>r)$ goes to 0 as $r\rightarrow\infty$ exponentially fast. Moreover,  letting $Y=|U|$, we have that 
    \begin{align}
        \mathbb{E}[\mathbf{1}_{\{|U|>r\}}|U|]&=\mathbb{E}[\mathbf{1}_{\{Y>r\}}Y]\notag\\
        &=\int_{r}^\infty \frac{y}{\sqrt{2\pi}}\left(2\exp\left(-\frac{y^2}{2}\right)\right)dy\notag\\
                &=\sqrt{\frac{2}{\pi}}\exp(-\frac{r^2}{2})\notag,
    \end{align}
    which also converges to 0 as $r\rightarrow\infty$.
    Therefore, the right-hand side of \eqref{eqn:proof_new_0} converges to 0 as $r\rightarrow\infty$ and thus can be bounded above by $\epsilon$ for a large enough $r$.

    Consider the first term. It is bounded above by
    \begin{align}
 \left|\hat{\alpha}_I(z)-\hat{\alpha}^{(r)}_I(z)\right|&\leq\left|\frac{1}{I}\sum_{i=1}^I\mathbf{1}_{\{|U_i|>r\}}\left(\hat{\beta}_I(U_i)-\hat{\beta}_I(U_i^{(r)})\right)\right|\notag\\
                &\leq\frac{1}{I}\sum_{i=1}^I\mathbf{1}_{\{|U_i|>r\}}\left(|\hat{\beta}_I(U_i)|+|\hat{\beta}_I(U_i^{(r)})|\right)\notag\\
                &=\frac{1}{I}\sum_{i=1}^I\mathbf{1}_{\{|U_i|>r\}}\frac{1}{J(I)}\sum_{j=1}^{J(I)}a(U_i,W_j)+a(U^{(r)}_i,W_j)\label{eqn:proof_new_1}
    \end{align}
    By the strong law of large number the term on the right hand side of \eqref{eqn:proof_new_1} converges almost surely to the term on the right hand side of \eqref{eqn:proof_new_0}. Since we have chosen $r$ large enough such that the right hand side of \eqref{eqn:proof_new_0} is bounded by $\epsilon$, we can choose $I$ large enough such that the right hand side of 
    \eqref{eqn:proof_new_1} is also bounded by $\epsilon.$ 

    Therefore, for a given $\epsilon>0$, w.p. 1,
    $$|\hat{\alpha}_{I,J(I)}(z)-\alpha(z)|<3\epsilon,$$
    for large enough $I$. This implies $\lim_{I\rightarrow\infty}\hat{\alpha}_{I,J(I)}(z)=\alpha(z)$ almost surely as desired.
\end{proof}
\begin{lemma}
\label{lemma:unicon}
    $\alpha(z)$ is continuous in $z$ and $\hat{\alpha}_{I,J(I)}(z)$ converges uniformly to $\alpha(z)$ on $\mathbb{Z}$ as $I\rightarrow\infty$ and $\lim_{I\rightarrow\infty}J(I)=\infty.$
\end{lemma}
\begin{proof}[Proof of Lemma~\ref{lemma:unicon}]
    Observe that
    \begin{align}
        \max_{x\in\mathbb{X}}\mathbb{E}[h(U,W,x,z)|U]&\leq\max_{x\in\mathbb{X}}\mathbb{E}[a(U,W)|U]=\mathbb{E}[a(U,W)|U],\label{eqn:proof1}
    \end{align}
    where the inequality follows from Lemma~\ref{lemma:dominated}. The term of the right hand side of \eqref{eqn:proof1} is an integrable function of $U$ since
    \begin{align}
        \mathbb{E}[|\mathbb{E}[a(U,W)|U|]=\mathbb{E}[\mathbb{E}[a(U,W)|U]]=\mathbb{E}[a(U,W)]<\infty.\label{eqn:proof2}
    \end{align}
    Consider any $z\in\mathbb{Z}$ and a sequence $\{z_k:k\geq 1\}\subseteq \mathbb{Z}$ converging to $z$. By the Lebesgue dominated convergence theorem and \eqref{eqn:proof2},
    \begin{align}
        \lim_{k\rightarrow\infty}\mathbb{E}[\max_{x\in\mathbb{X}}\mathbb{E}[h(U,W,x,z_k)|U]]&=\mathbb{E}[\lim_{k\rightarrow\infty}\max_{x\in\mathbb{X}}\mathbb{E}[h(U,W,x,z_k)|U]]\label{eqn:proof_new}
    \end{align}
    Fixing $U$, let $g(x,z)=\mathbb{E}[h(U,W,x,z)|U]$. This function is continuous jointly in $x,z$ because for an arbitrary sequence $\{(x_k,z_k):k\geq 1\}$ converging to $(x,z)$, we have
    
    $$\lim_{k\rightarrow\infty}\mathbb{E}[h(U,W,x_k,z_k)]=\mathbb{E}[\lim_{k\rightarrow\infty}h(U,W,x_k,z_k)]=\mathbb{E}[h(U,W,x,z)],$$
    where we have used Lemma~\ref{lemma:dominated} and the Lebesgue dominated convergence theorem in the first equality and used continuity (Lemma~\ref{lemma:conth}) in the second equality. 

    Since $\mathbb{X}\times\mathbb{Z}$ is compact, $g$ is uniformly continuous over  $\mathbb{X}\times\mathbb{Z}$. Thus, there exists a large enough $K$ such that 
    $$|g(x,z_k)-g(x,z)|<\epsilon,$$
    for all $k\geq K$ and all $x\in\mathbb{X}.$ Let $x_k^*\in\arg\max_{x\in\mathbb{X}} g(x,z_k)$. For all such $k$, we have
    \begin{equation}
        \max_{x\in\mathbb{X}}g(x,z_k)-\max_{x\in\mathbb{X}}g(x,z)\leq g(x_k^*,z_k)-g(x_k^*,z)<\epsilon.\label{eqn:proof0.1}
    \end{equation}
    Similarly, let $x^*\in\arg\max_{x\in\mathbb{X}}g(x,z)$. For all such $k$, we have
    \begin{equation}
        -\epsilon<g(x^*,z_k)-g(x^*,z)\leq\max_{x\in\mathbb{X}}g(x,z_k)-\max_{x\in\mathbb{X}}g(x,z)\label{eqn:proof0.2}
    \end{equation}
    Equations~\eqref{eqn:proof0.1} and~\eqref{eqn:proof0.2} imply
    $$\left|\max_{x\in\mathbb{X}}g(x,z_k)-\max_{x\in\mathbb{X}}g(x,z)\right|<\epsilon.$$
    This shows that $\lim_{k\rightarrow\infty}\max_{x\in\mathbb{X}}g(x,z_k)=\max_{x\in\mathbb{X}}g(x,z).$ Thus, \eqref{eqn:proof_new} becomes
    \begin{align}
        \lim_{k\rightarrow\infty}\mathbb{E}[\max_{x\in\mathbb{X}}\mathbb{E}[h(U,W,x,z_k)|U]]&=\mathbb{E}[\lim_{k\rightarrow\infty}\max_{x\in\mathbb{X}}\mathbb{E}[h(U,W,x,z_k)|U]]\notag\\
        &=\mathbb{E}[\max_{x\in\mathbb{X}}\mathbb{E}[h(U,W,x,z)|U]]=\alpha(z).\label{eqn:proof_new2}
    \end{align}
    Equation~\eqref{eqn:proof_new2} implies $\lim_{k\rightarrow\infty}\alpha(z_k)=\alpha(z)$, showing that $\alpha(\cdot)$ is continuous on $\mathbb{Z}$.

    Now consider a sequence $\mathbb{N}_k$ of neighborhoods of a generic point $z\in\mathbb{Z}$ shrinking to $\{z\}$. Let 
    \begin{equation}
        b_k(u,w) = \sup\{|h(u,w,x,z)-h(u,w,x,y)|:y\in \mathbb{N}_k, x\in\mathbb{X}\}.\notag
    \end{equation}
    Recall that $h(u,w,\cdot, \cdot)$ is continuous for almost every $(u,w)$. Choose $u,w$ such that $h(u,w,\cdot,\cdot)$ is continuous on $\mathbb{X}\times\mathbb{Z}$. Since $\mathbb{X}\times \mathbb{Z}$ is compact, $h(u,w,\cdot,\cdot)$ is uniformly continuous on $\mathbb{X}\times\mathbb{Z}$.
    Therefore, given $\epsilon>0$, there exists $\delta>0$ such that $|h(u,w,x',z')-h(u,w,x'',z'')|<\epsilon$ when $||(x',z')-(x'',z'')||<\delta$. Thus, for a large enough $k$, we have $|h(u,w,x,z)-h(u,w,x,y)|<\epsilon$ for all $y\in \mathbb{N}_k$. This implies $\lim_{k\rightarrow\infty}b_k(u,w)=0$ for almost every $(u,w)$.

    Define 
    \begin{equation}
        c_k(U_i,W_{1:J})=\sup\left\{\left|\max_{x\in\mathbb{X}}\frac{1}{J}\sum_{j=1}^Jh(U_i,W_j,x,z)-\max_{x\in\mathbb{X}}\frac{1}{J}\sum_{j=1}^Jh(U_i,W_j,x,z')\right|:z'\in \mathbb{N}_k\right\}.\notag
    \end{equation}
    Let $x^*(z')\in\arg\max_{x\in\mathbb{X}}\frac{1}{J}\sum_{j=1}^Jh(U_i,W_j,x,z')$. For $z'\in\mathbb{N}_k$,
    \begin{align}
        \max_{x\in\mathbb{X}}\frac{1}{J}\sum_{j=1}^Jh(U_i,W_j,x,z)&-\max_{x\in\mathbb{X}}\frac{1}{J}\sum_{j=1}^Jh(U_i,W_j,x,z')\notag\\
        &\leq\frac{1}{J}\sum_{j=1}^Jh(U_i,W_j,x^*(z),z)-\frac{1}{J}\sum_{j=1}^Jh(U_i,W_j,x^*(z),z')\notag\\
        &=\frac{1}{J}\sum_{j=1}^J[h(U_i,W_j,x^*(z),z)-h(U_i,W_j,x^*(z),z')]\notag\\
        &\leq\frac{1}{J}\sum_{j=1}^J|h(U_i,W_j,x^*(z),z)-h(U_i,W_j,x^*(z),z')|\notag\\
        &\leq\frac{1}{J}\sum_{j=1}^Jb_k(U_i,W_j).\label{eqn:new_ineqbk1}
    \end{align}
    Similarly,
        \begin{align}
        \max_{x\in\mathbb{X}}\frac{1}{J}\sum_{j=1}^Jh(U_i,W_j,x,z)&-\max_{x\in\mathbb{X}}\frac{1}{J}\sum_{j=1}^Jh(U_i,W_j,x,z')\notag\\
        &\geq\frac{1}{J}\sum_{j=1}^Jh(U_i,W_j,x^*(z'),z)-\frac{1}{J}\sum_{j=1}^Jh(U_i,W_j,x^*(z'),z')\notag\\
        &\geq-\frac{1}{J}\sum_{j=1}^J|h(U_i,W_j,x^*(z'),z)-h(U_i,W_j,x^*(z'),z')|\notag\\
        &\geq-\frac{1}{J}\sum_{j=1}^Jb_k(U_i,W_j).\label{eqn:new_ineqbk2}
    \end{align}
    Equations~\eqref{eqn:new_ineqbk1} and~\eqref{eqn:new_ineqbk2} imply 
    $$\left|\max_{x\in\mathbb{X}}\frac{1}{J}\sum_{j=1}^Jh(U_i,W_j,x,z)-\max_{x\in\mathbb{X}}\frac{1}{J}\sum_{j=1}^Jh(U_i,W_j,x,z')\right|\leq \frac{1}{J}\sum_{j=1}^Jb_k(U_i,W_j).$$
    This implies 
    $$ c_k(U_i,W_{1:J})\leq \frac{1}{J}\sum_{j=1}^J b_k(U_i,W_j).$$
    Thus for all $z'\in\mathbb{N}_k$,
    \begin{align}
        |\hat{\alpha}_{I,J(I)}(z')-\hat{\alpha}_{I,J(I)}(z)|&\leq\frac{1}{I}\sum_{i=1}^Ic_k(U_i,W_{1:J})\notag\\
        &\leq \frac{1}{I}\sum_{i=1}^I\frac{1}{J}\sum_{j=1}^Jb_k(U_i,W_j).\label{eqn:proof4.0}
    \end{align}
    Fix $J$ to be a function of $I$, i.e. $J=J(I)$ where $\lim_{I\rightarrow\infty}J(I)=\infty$. By the strong law of large number, the right hand side of Equation~\eqref{eqn:proof4.0} converges to $\mathbb{E}[b_k(U,W)]$ as $I\rightarrow\infty$ almost surely.

    Lemma~\ref{lemma:dominated} implies $b_k(U,W)\leq 2a(U,W)$. Thus, by the Lebesgue dominated convergence theorem, $\lim_{k\rightarrow\infty}\mathbb{E}[b_k(U,W)]=\mathbb{E}[\lim_{k\rightarrow\infty}b_k(U,W)]=0$.

    For any $\epsilon>0$, there is a large enough $K$ such that $\mathbb{E}[b_K(U,W)]<\epsilon$. Moreover, w.p.1, there is a large enough $I^*$ such that 
    $$\sup_{z'\in\mathbb{N}_K}|\hat{\alpha}_{I,J(I)}(z')-\hat{\alpha}_{I,J(I)}(z)|<\epsilon\hspace{0.5cm}\forall I\geq I^*.$$
    Moreover, since $\alpha$ is continuous, it is possible to choose $K$ large enough that 
    $$\sup_{z'\in\mathbb{N}_K}|\alpha(z')-\alpha(z)|<\epsilon.$$
    We have shown that, for any point $z\in\mathbb{Z}$ and any $\epsilon>0$, there is a neighborhood $\mathbb{N}(z,\epsilon)$ of $z$ such that w.p. 1 there is a large enough $I^*$ such that
    $$\sup_{z'\in\mathbb{N}(z,\epsilon)}|\hat{\alpha}_{I,J(I)}(z')-\hat{\alpha}_{I,J(I)}(z)|<\epsilon\hspace{0.5cm}\forall I\geq I^*$$
    and that 
    $$\sup_{z'\in \mathbb{N}(z,\epsilon)}|\alpha(z')-\alpha(z)|<\epsilon.$$

    Let $\{\mathbb{N}(z,\epsilon):z\in\mathbb{Z}\}$ be an open cover of $\mathbb{Z}$. Since $\mathbb{Z}$ is compact, we can choose a finite subcover. This gives a collection of points $z_1,\hdots,z_J$ with neighborhoods $\mathbb{N}(z_j,\epsilon)$ that cover $\mathbb{Z}$ such that w.p. 1, for sufficiently large $I$,
    $$\sup\{|\hat{\alpha}_{I,J(I)}(z')-\hat{\alpha}_{I,J(I)}(z_j)|:z'\in \mathbb{N}(z_j,\epsilon)\}<\epsilon$$
    and
    $$\sup\{|\alpha(z')-\alpha(z_j)|:z'\in\mathbb{N}(z_j,\epsilon)\}<\epsilon.$$
    By Lemma~\ref{lemma:pointwise}, we have
    $$|\hat{\alpha}_{I,J(I)}(z_j)-\alpha(z_j)|<\epsilon.$$
    Thus, given $\epsilon>0$, w.p. for $I$ large enough, 
    \begin{align}
        |\hat{\alpha}_{I,J(I)}(z)-\alpha(z)|\leq |\hat{\alpha}_{I,J(I)}(z)-\hat{\alpha}_{I,J(I)}(z_j)|+|\hat{\alpha}_{I,J(I)}(z_j)-\alpha(z_j)|+|\alpha(z_j)-\alpha(z)|<3\epsilon,\notag
    \end{align}
    where $z\in\mathbb{N}(z_j,\epsilon).$ This implies that $\hat{\alpha}_{I,J(I)}(z)\rightarrow\alpha(z)$ uniformly on $\mathbb{Z}$ as $I\rightarrow\infty$, as desired.
\end{proof}

Using the above lemmas, we are now in position to prove Theorem~\ref{thm:mainthm}.
\begin{proof}[Proof of Theorem~\ref{thm:mainthm}]
    From Lemma~\ref{lemma:unicon}, we know that for any $\epsilon>0$ and a sufficiently large $I$ and all $z\in\mathbb{Z}$,
    \begin{equation}
        |\hat{\alpha}_{I,J(I)}(z)-\alpha(z)|<\epsilon \hspace{0.2cm}\text{w.p. 1}.\label{eqn:continuousalpha}
    \end{equation}
    Let $\hat{z}_{I,J(I)}\in\arg\max_{z\in\mathbb{Z}}\hat{\alpha}_{I,J(I)}(z)$ and $z^*\in\arg\max_{z\in\mathbb{Z}}\alpha(z)$. It follows that 
    $$|\hat{\alpha}_{I,J(I)}(\hat{z}_{I,J(I)})-\alpha(\hat{z}_{I,J(I)})|<\epsilon,$$
    which implies 
    \begin{equation}
        \hat{\alpha}_{I,J(I)}(\hat{z}_{I,J(I)})<\alpha(\hat{z}_{I,J(I)})+\epsilon<\alpha(z^*)+\epsilon.\label{eqn:proof4}
    \end{equation}
    Similarly, we have
    $$|\hat{\alpha}_{I,J(I)}(z^*)-\alpha(z^*)|<\epsilon,$$
    which implies 
    \begin{equation}
        -\epsilon+\alpha(z^*)<\hat{\alpha}_{I,J(I)}(z^*)<\hat{\alpha}_{I,J(I)}(\hat{z}_{I,J(I)})\label{eqn:proof5}
    \end{equation}
    Equations~\eqref{eqn:proof4} and~\eqref{eqn:proof5} yield $|\hat{\varphi}_{I,J(I)}-\varphi^*|<\epsilon$. Since $\epsilon$ is arbitrary, this implies $\hat{\varphi}_{I,J(I)}$ converges to $\varphi^*$ almost surely as desired.

    Furthermore, suppose that $z^*$ is a unique maximizer of $\alpha(z)$ over $\mathbb{Z}$. Consider any neighborhood $\mathbb{N}$ of $z^*\in\mathbb{Z}$. Following from the continuity of $\alpha$ and the compactness of $\mathbb{Z}$, there exists $\epsilon>0$ such that 
    $$\alpha(z)>\alpha(z^*)+2\epsilon,$$
    for all $z\in\mathbb{Z}$ and $z\not\in\mathbb{N}$. This combines with \eqref{eqn:continuousalpha} implies 
    \begin{equation}
        \hat{\alpha}_{I,J(I)}(z)>\alpha(z^*)+\epsilon,
    \end{equation}
    for all $z\in\mathbb{Z}$ and $z\not\in\mathbb{N}$. Also, since we have $|\hat{\varphi}_{I,J(I)}-\varphi^*|<\epsilon$, this implies
    \begin{equation}
        \hat{\alpha}_{I,J(I)}(\hat{z}_{I,J(I)})=\hat{\varphi}_{I,J(I)}<\varphi^*+\epsilon=\alpha(z^*)+\epsilon.\label{eqn:solncon}
    \end{equation}
    Equation \eqref{eqn:solncon} implies $\hat{z}_{I,J(I)}\in\mathbb{N}$. Since $\mathbb{N}$ can be chosen arbitrarily, it implies $\hat{z}_{I,J(I)}\rightarrow z^*$ almost surely.

    Now, let us consider the case where $\alpha$ has multiple maximizers over $\mathbb{Z}$. Suppose for the sake of contradiction that $d(\hat{z}_{I,J(I)},\mathcal{Z}^*):=\inf_{z\in\mathcal{Z}^*}||\hat{z}_{I,J(I)}-z||\nrightarrow0$. Since $\mathbb{Z}$ is compact, there exists a sequence $\{\hat{z}_{I,J(I)}:I\geq 1\}$ such that $d(\hat{z}_{I,J(I)},\mathcal{Z}^*)\geq \epsilon$ for some $\epsilon>0$ and the sequence $\{\hat{z}_{I,J(I)}:I\geq 1\}$ converges to a point $z'\in \mathbb{Z}$, but not in $\mathcal{Z}^*$. By the continuity of $\alpha$, we have that $\alpha(\hat{z}_{I,J(I)})\rightarrow\alpha(z')<\varphi^*$ as $I\rightarrow\infty$. Moreover, $\hat{\varphi}_{I,J(I)}=\hat{\alpha}_{I,J(I)}(\hat{z}_{I,J(I)})$ and
    \begin{equation}
        \hat{\alpha}_{I,J(I)}(\hat{z}_{I,J(I)})-\alpha(z')=(\hat{\alpha}_{I,J(I)}(\hat{z}_{I,J(I)})-\alpha(\hat{z}_{I,J(I)}))+(\alpha(\hat{z}_{I,J(I)})-\alpha(z')).\label{eqn:part2thm1.1}
    \end{equation}
    The first term on the right-hand side of \eqref{eqn:part2thm1.1} goes to zero as $I\rightarrow\infty$ by Lemma \ref{lemma:unicon}, and the second term also converges to 0 by the continuity of $\alpha$. This implies that $\hat{\varphi}_{I,J(I)}\rightarrow\alpha(z')<\varphi^*$, which is a contradiction. Thus, we have $d(\hat{z}_{I,J(I)},\mathcal{Z}^*)\rightarrow0$ almost surely, as desired.
\end{proof}

\section{Additional Details on Acquisition Function Optimization}\label{appdx:opt}

We describe model configuration and hyperparameters used to optimize the acquisition functions used in our numerical experiments. We also present a comparison between the one-shot optimization and optimization-via-discretization approaches for optimizing the p-KGFN acquisition.

\subsection{Hyperparameters in Acquisition Function Optimization}

In our experiments, all methods utilize independent GPs with zero mean functions and the Mate\'rn 5/2 kernel \citepAP{genton2001classes}, with automatic relevance determination (ARD). 
The lengthscales of the GPs are assumed to have Gamma priors: for the Ackley, FreeSolv and Pharm problems, Gamma(3, 6); for the Manu-GP problem, Gamma(5, 2). The outputscale parameters are assumed to have Gamma(2, 0.15) priors in all problems.
The lengthscales and outputscales are then estimated via maximum a posteriori (MAP) estimation.

In p-KGFN, we estimate the posterior mean of a function network's final node value with $J= 64$ quasi-MC samples using Sobol sequences \citepAP{balandat2020botorch}. For EIFN, we follow the implementation in \citetAP{astudillo2021bayesian}, using $J=128$.

To compute MC estimates for the p-KGFN acquisition value, we use $I = 8$ fantasy models.  
As described in Section~\ref{subsec:discretization}, we optimize the p-KGFN acquisition function via discretization, replacing the domain of the optimization problem in line 6 of Algorithm~\ref{algo:KGFN} by a discrete set of candidate solutions $\mathbb{A}$. We include in $\mathbb{A}$ the current maximizer of the final node posterior mean $\bmx_n^*$, $N_T = 10$ points obtained from a Thompson sampling method described in the main text, $N_L = 10$ local points with $r=0.1$.

When optimizing EIFN, TSFN, KGFN with full evaluations or p-KGFN for a problem with $d$-dimensional function network input, we use $100d$ raw samples for identifying the starting points for the multi-start acquisition optimization and $10d$ starting points for the multi-start acquisition function optimization. We set the number of raw samples to 100 and the number of starting points to 20 when optimizing the standard EI and KG acquisition functions.

We refit the hyperparameters of the GP models in each iteration. For p-KGFN, this occurs after we obtain one additional observation of a specific function node; for other benchmarks, this occurs after we obtain observations of the entire function network for a design point.

All algorithms are implemented in the open source BoTorch package \citepAP{balandat2020botorch}. We extend the implementation outlined in \citetAP{astudillo2021bayesian} to enable partial evaluations for function networks. 

\subsection{Comparison Between One-Shot Optimization and Optimization-Via-Discretization}
\label{appdx:OSvsDiscrete}
\begin{figure}[h!]
    \centering
    \includegraphics[width=1.0\textwidth]{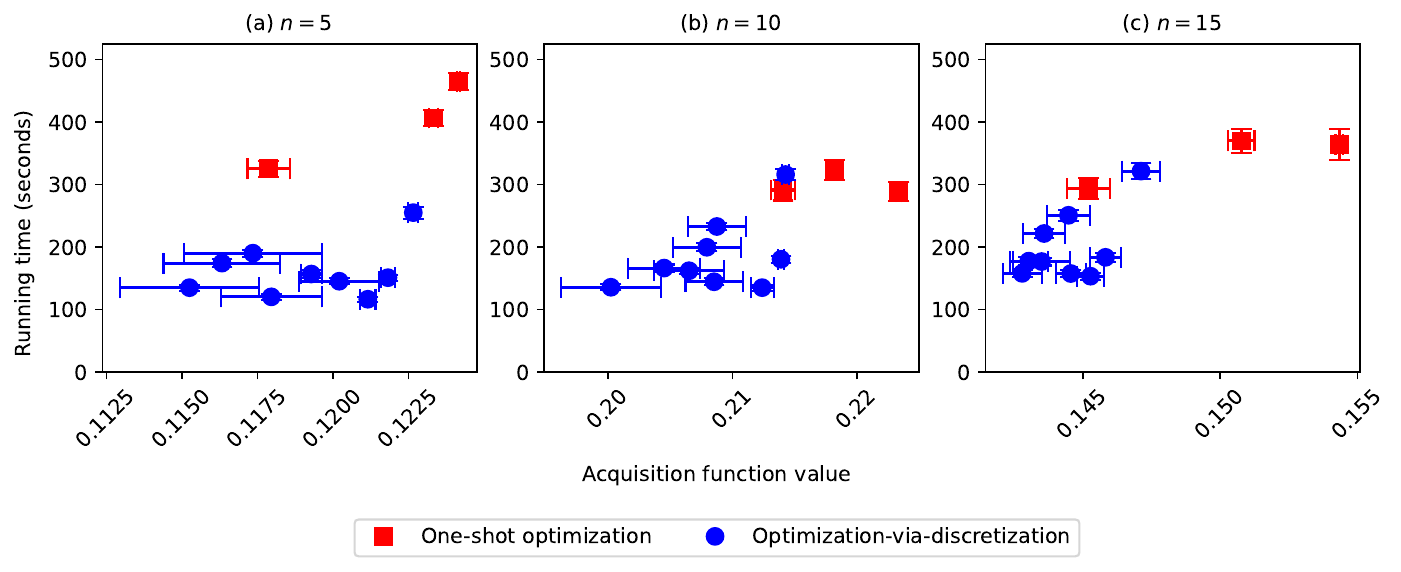}
    \caption{Comparison between one-shot optimization and optimization-via-discretization in terms of acquisition value attained and computation time, on the Ackley6D problem with different numbers of fully-evaluated points: (a) $n=5$, (b) $n=10$, and (c) $n=15$. Both the value averaging over 100 trials and the standard error are presented. }
    \label{fig:OsDisC}
\end{figure}
We perform a comparative analysis between two methods for optimizing KG-based acquisition functions discussed in Section~\ref{subsec:discretization}: ``One-shot'' optimization and optimization-via-discretization. 

The one-shot optimization method, proposed by \citetAP{balandat2020botorch}, has gained popularity in optimizing non-myopic acquisition functions in recent years \citepAP{jiang2020efficient, astudillo2021multi, daulton2023hvkg}. This method effectively addresses the nested optimization problem, a common challenge encountered when optimizing KG. However, it introduces its own challenges -- namely, by turning the nested optimization problem into a single high-dimensional (and typically more difficult) optimization problem, the one-shot optimization approach can be more likely to get stuck in local optima and require long computation time (due to the higher number of optimization variables, and the potential of the numerical optimization failing to converge).

Optimization-via-discretization, as used by \citetAP{frazier2009knowledge}, is a commonly used alternative technique for optimizing KG \citepAP{ungredda2022efficient, buckingham2023bayesian}. This method accelerates KG optimization by discretizing the domain of the inner optimization problem. Nevertheless, it has its own limitations, as the discretization will lead to less accurate solutions for the inner optimization problem, potentially resulting in sub-optimal outcomes, especially in higher dimensions (due to the coarser coverage of the space).

For p-KGFN, we follows the optimization-via-discretization approach. However, rather than selecting points randomly within the domain or arranging them in a grid, we intelligently select the points that form the discrete set for the inner optimization problem (see Section~\ref{subsec:discretization}).

To better understand the performance of one-shot optimization and optimization-via-discretization in the context of optimizing p-KGFN, we conduct a comparative analysis. Specifically, we examine both approaches under different configurations using the Ackley6D test problem. We consider problem instances where we have access 5, 10, and 15 full evaluations of the function network. These observations are held fixed and used to fit GP network models. We seek to optimize p-KGFN at the first function node. 

For both approaches, we explore three choices ($I=2, 4, 8$) for the number of fantasy models for estimating the p-KGFN acquisition values. In the optimization-via-discretization approach, we generate the discrete set by using Thompson sampling approach, i.e., sampling $N_T$ realizations of a function network based on the current posterior distributions and optimize the realizations to obtain their maximizers to include in the discrete set $\mathbb{A}$, and randomly choosing $N_L$ local points close to the current posterior mean maximizer $x_n^*$. The current posterior mean maximizer itself is also included in the set. We consider three discrete set sizes for the inner optimization problem: $|\mathbb{A}|=11$ ($N_T=N_L=5$), $|\mathbb{A}|=21$ ($N_T=N_L=10$), and $|\mathbb{A}|=41$ ($N_T=N_L=20$).
We measure performance in terms of both true acquisition function value of the selected candidates (approximated by the MC estimator described in Algorithm~\ref{algo:KGFN}, with $I=512$ and $J=128$) and computational time (measured in seconds), averaging over 100 trials.

In Figure~\ref{fig:OsDisC} we observe that although one-shot optimization demonstrates the best performance in terms of acquisition value attained, optimization-via-discretization can attain comparable acquisition value at significantly reduced run times. 
For the setup used in our experiments ($I=8$ and $|\mathbb{A}| = 21$), optimization-via-discretization achieves slightly lower acquisition value compared to the best-performing one-shot optimization method with $I=8$ (1.89\%, 4.25\%, and 5.51\% lower for the problem instances with 5, 10, and 15 fully evaluated points, respectively).
However, it significantly reduces the average run time compared to the best-performing one-shot optimization algorithm, by 67.5\%, 37.4\%, and 50.5\%.
The significant reduction in run time, along with the marginal loss in acquisition values, justifies the employment of the discretization approach in our experiments.
The full comparison analysis results are summarized in Table~\ref{tab:comparison}.

\begin{table}[]
\caption{Comparisons between one-shot optimization and optimization-via-discretization on the Ackley6D test problem, with 5, 10, and 15 design points fully evaluated across the network. Performances are reported in terms of acquisition function value obtained and running time. Both the value averaging over 100 trials and the standard error are reported. Numbers in boldface denote the highest acquisition function value and the shortest running time.}\label{tab:comparison}\vspace{0.1in}
\resizebox{\textwidth}{!}{%
\begin{tabular}{ccccc}
\multicolumn{5}{c}{(a) Number of fully evaluated points: 5}   \\ \toprule
\textbf{Method} & \textbf{Num of fantasy models $I$} & \textbf{Size of discrete set $|\mathbb{A}|$} & \textbf{Avg. acqf. val.} & \textbf{Avg. running time (seconds)} \\ \midrule
\multirow{3}{*}{One-shot} & 2 & NA & $0.11787 \pm 0.00070$ & $325 \pm 13$ \\
 & 4 & NA & $0.12332 \pm 0.00016$ & $406 \pm 13$ \\
 & 8 & NA & \textbf{$\mathbf{0.12416 \pm 0.00004}$} & $465 \pm 14$ \\ \midrule
\multirow{9}{*}{Discretization} & 2 & 11 & $0.11525 \pm 0.00231$ & $135 \pm 5\;$ \\
 & 2 & 21 & $0.11928 \pm 0.00035$ & $157 \pm 6\;$ \\
 & 2 & 41 & $0.1164 \pm 0.00194$ & $174 \pm 6\;$ \\
 & 4 & 11 & $0.11795 \pm 0.00166$ & $120 \pm 5\;$ \\
 & 4 & 21 & $0.12020 \pm 0.00133$ & $145 \pm 5\;$ \\
 & 4 & 41 & $0.11734 \pm 0.00226$ & $190 \pm 6$ \\
 & 8 & 11 & $0.12115 \pm 0.00026$ & $\mathbf{117 \pm 4}$ \\
 & 8 & 21 & $0.12181 \pm 0.00024$ & $151 \pm 5\;$ \\
 & 8 & 41 & $0.12265 \pm 0.00017$ & $255 \pm 10$ \\ \bottomrule
 &  &  &  & \\

\multicolumn{5}{c}{(b) Number of fully evaluated points: 10} \\ \toprule
\textbf{Method} & \textbf{Num of fantasy models $I$} & \textbf{Size of discrete set $|\mathbb{A}|$} & \textbf{Avg. acqf. val.} & \textbf{Avg. running time (seconds)} \\ \midrule
\multirow{3}{*}{One-shot} & 2 & NA & $0.2141 \pm 0.0010$ & $291 \pm 17$ \\
 & 4 & NA & $0.2182 \pm 0.0006$ & $324 \pm 16$ \\
 & 8 & NA & \textbf{$\mathbf{0.2234 \pm 0.0002}$} & $289 \pm 16$ \\ \midrule
\multirow{9}{*}{Discretization} & 2 & 11 & $0.2085 \pm 0.0023$ & $145 \pm 5\;$ \\
 & 2 & 21 & $0.2045 \pm 0.0029$ & $166 \pm 6\;$ \\
 & 2 & 41 & $0.2079 \pm 0.0027$ & $200 \pm 7\;$ \\
 & 4 & 11 & $0.2002 \pm 0.0040$ & $136 \pm 5\;$ \\
 & 4 & 21 & $0.2065 \pm 0.0028$ & $162 \pm 5\;$ \\
 & 4 & 41 & $0.2087 \pm 0.0023$ & $233 \pm 6$ \\
 & 8 & 11 & $0.2124 \pm 0.0009$ & $\mathbf{135 \pm 5}$ \\
 & 8 & 21 & $0.2139 \pm 0.0003$ & $181 \pm 5\;$ \\
 & 8 & 41 & $0.2143 \pm 0.0003$ & $316 \pm 9$ \\ \bottomrule
 &  &  &  & \\ 

\multicolumn{5}{c}{(c) Number of fully evaluated points: 15} \\ \toprule
\textbf{Method} & \textbf{Num of fantasy models $I$} & \textbf{Size of discrete set $|\mathbb{A}|$} & \textbf{Avg. acqf. val.} & \textbf{Avg. running time (seconds)} \\ \midrule
\multirow{3}{*}{One-shot} & 2 & NA & $0.1452 \pm 0.0008$ & $294 \pm 17$ \\
 & 4 & NA & $0.1508 \pm 0.0005$ & $370 \pm 19$ \\
 & 8 & NA & \textbf{$\mathbf{0.1543 \pm 0.0001}$} & $364 \pm 24$ \\ \midrule
\multirow{9}{*}{Discretization} & 2 & 11 & $0.1428 \pm 0.0007$ & $158 \pm 6\;$ \\
 & 2 & 21 & $0.1430 \pm 0.0006$ & $177 \pm 6\;$ \\
 & 2 & 41 & $0.1436 \pm 0.0008$ & $222 \pm 8\;$ \\
 & 4 & 11 & $0.1446 \pm 0.0005$ & $157 \pm 6\;$ \\
 & 4 & 21 & $0.1435 \pm 0.0010$ & $177 \pm 6\;$ \\
 & 4 & 41 & $0.1445 \pm 0.0008$ & $250 \pm 9$ \\
 & 8 & 11 & $0.1453 \pm 0.0005$ & $\mathbf{153 \pm 6}$ \\
 & 8 & 21 & $0.1458 \pm 0.0006$ & $184 \pm 7\;$ \\
 & 8 & 41 & $0.1471 \pm 0.0007$ & $321 \pm 13$ \\ \bottomrule
\end{tabular}%
}
\end{table}

\section{Additional Details on Numerical Experiments}\label{appdx:experiments}

\subsection{Ackley6D (Ackley)}
We design the Ackley synthetic test problem as a two-stage function network (Figure~\ref{fig:eggandfree}). The first function node is the 6-dimensional negated Ackley function \citepAP{ackley2012connectionist}: 

\begin{equation*}
    f_1(x)=20\exp\left(-0.2\sqrt{\frac{1}{6}\sum_{i=1}^6x_i^2}\right)+\exp\left(\frac{1}{6}\sum_{i=1}^6\cos{(2\pi x_i)}\right)-20-\exp(1),
\end{equation*}
where $x_i\in[-2,2]$ for $i=1,\hdots,6$. The second function node, which takes as an input the output of the first function node, is defined as follows
\begin{equation*}
    f_2(y)=-y\sin{\left(\frac{5y}{6\pi}\right)}.
\end{equation*}

\subsection{Manufacturing Network (Manu-GP)}
Motivated by real-world manufacturing applications where several intermediate parts are produced and then assembled to create a final product, we formulate this second test problem involving two processes, each of which takes a two-dimensional input (Figure~\ref{fig:manu}). The first process consists of two sequential sub-processes, denoted as $f_1$ and $f_2$. The first sub-process takes $x_1,x_2$ as inputs and returns $y_1$ which is taken by the second sub-process to produce an output $y_2$. The second process, $f_3$, takes $x_3$ and $x_4$ as input and produces an output $y_3$. The outputs $y_2$ and $y_3$ are then combined in the final process $f_4$ to produce a final output $y_4$, which we aim to maximize. This experiment is designed to mimic a manufacturing scenario. Here, we set each input $x_i\in[-1,1]$, for $i=1,\ldots,4$. 

For each function node, we draw a sample path from a GP prior with the Mat\'ern 5/2 kernel \citepAP{genton2001classes}. Notably, we choose different lengthscales in the kernels for difference function nodes: 0.631 for $f_1$, 1 for $f_2$, 1 for $f_3$, and 3 for $f_4$. The outputscale parameter in the kernel is set to 0.631 for all functions but $f_4$, which has outputscale equal to 10. 
We show the drawn sample paths for each of the function nodes in Figure~\ref{fig:mugpfunc}.

\begin{figure}
    \centering
    \includegraphics[width=0.9\textwidth]{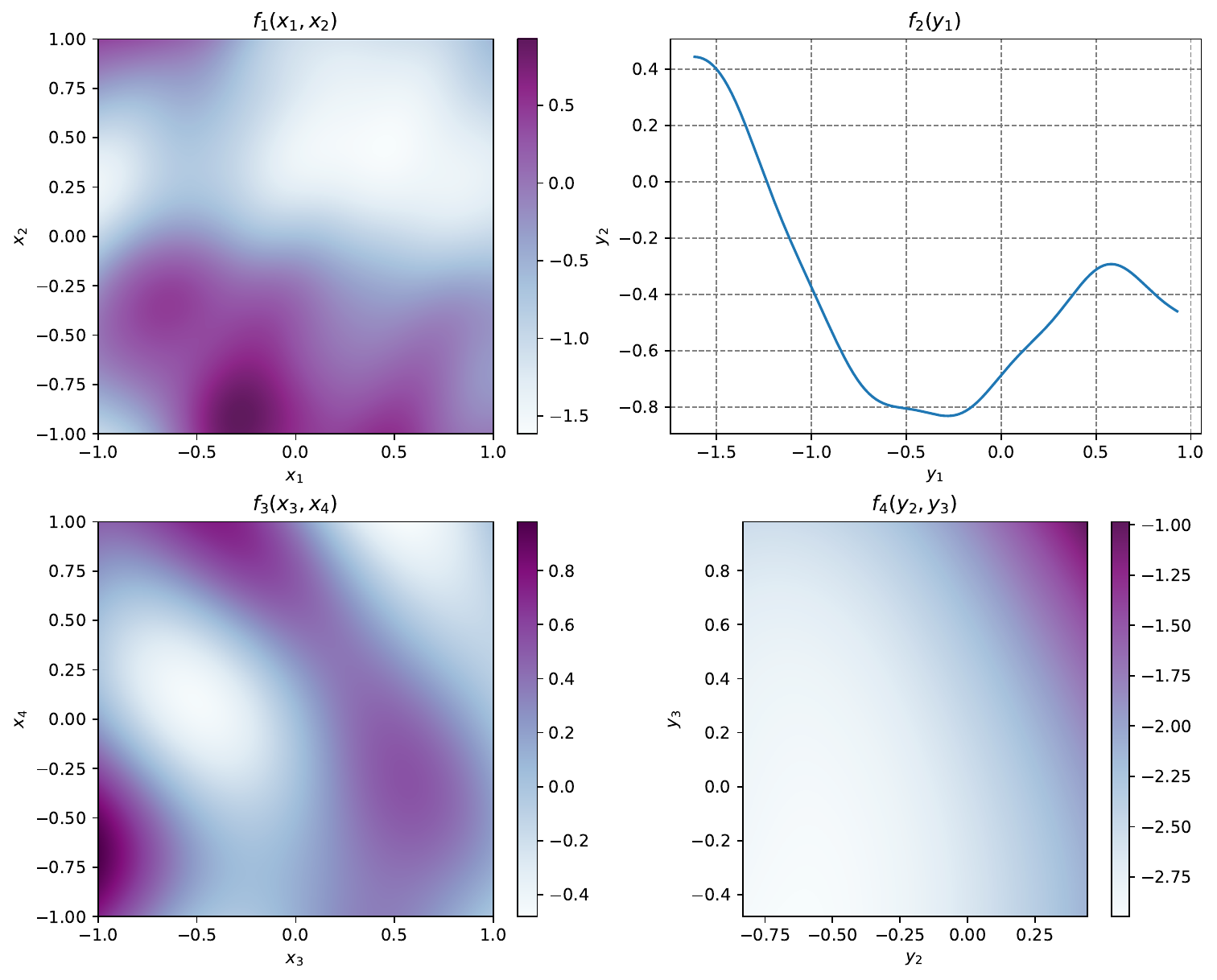}
    \caption{The sample paths drawn from GP priors for the manufacturing problem.}
    \label{fig:mugpfunc}
\end{figure}
\subsection{Molecular Design (FreeSolv)}
\label{appx:freesolv}
We use the FreeSolv dataset from \citetAP{mobley2014freesolv}, which comprises both calculated and experimental free energies (in kcal/mol) for 642 small molecules. 
Since lower free energy is preferred, we negate free energy (both calculated and experimental), setting our objective to maximizing the negative experimental free energy. 

To represent each small molecule in a continuous space, we utilize a variational autoencoder trained on the Zinc dataset as studied in \citetAP{gomez2018automatic}, resulting in a 196-dimensional representation in the unit cube, i.e., $x\in [0,1]^{196}$. We reduce the dimentionality of the representation to three through the standard principal component analysis (PCA) technique.

We then utilize the entire dataset to train two GP models, and use the posterior mean of the trained model as the underlying functions in the function network (Figure~\ref{fig:eggandfree}): 

\begin{enumerate}
    \item $f_1$ takes a three-dimensional representation as input and predicts the negative calculated free energy $y_1$;
    \item $f_2$ takes the negative calculated free energy as input and predicts the negative experimental free energy (Figure~\ref{fig:freesolv2}).
\end{enumerate}

\begin{figure}[h]
    \centering
    \includegraphics[width=0.5\textwidth]{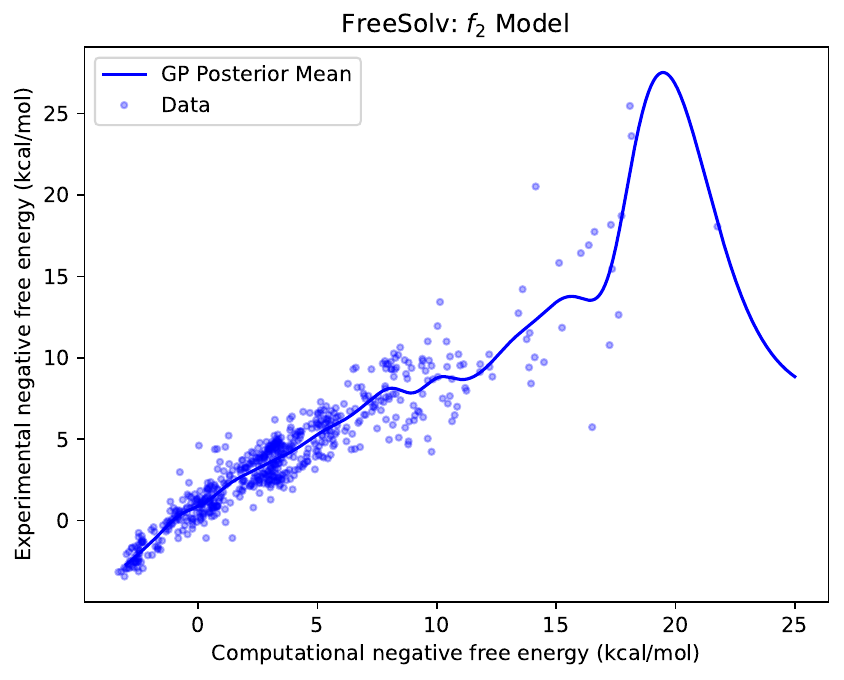}
    \caption{The posterior mean function of the GP fitted between calculated and experimental free energies.}
    \label{fig:freesolv2}
\end{figure}

\subsection{Pharmaceutical Product Development (Pharm)}
In this test problem, we tackle the challenge of manufacturing orally disintegration tablets (ODTs) that meet pharmaceutical standards, focusing on two crucial properties: disintegration time ($f_1$) and tensile strength ($f_2$). To model these properties, we adopt the neural network (NN) models proposed in \citetAP{sano2020application}. 

Specifically, the two properties are defined as functions of four parameters describing the production process, namely, $\beta$ form D-mannitol ratio in the total D-mannitol $(x_1)$, L-HPC ratio in a formulation $(x_2)$, granulation fluid level $(x_3)$ and compression force $(x_4)$. Each parameter $x_i$ is constrained to the range $[-1,1]$. The fitted NN models are shown as follows:
\begin{align*}
    f_1(x) &= -3.95 +9.20 \times (1+\exp(-(0.32+5.06x_1-4.07x_2-0.36 x_3-0.34\times x_4))^{-1} \\
    &\qquad+ 9.88 \times(1+\exp(-(-4.83+7.43x_1+3.46x_2+9.19x_3+16.58x_4)))^{-1}\\
    &\qquad+10.84\times(1+\exp(-7.90-7.91x_1-4.48x_2-4.08x_3-8.28x_4))^{-1}\\
    &\qquad+15.18\times(1+\exp(-(9.41-7.99x_1+0.65x_2+3.14x_3+0.31x_4)))^{-1}.
\end{align*}
and
\begin{align*}
f_2(x) &= 1.07 + 0.62 \times (1+\exp(-(3.05+0.03x_1 - 0.16x_2 + 4.03x_3-0.54 x_4)))^{-1}\\
&\qquad+0.65\times(1+\exp(-(1.78+0.60x_1 -3.19x_2 +0.10x_3+0.54x_4)))^{-1}\\
&\qquad-0.72\times(1+\exp(-(0.01+2.04x_1 -3.73x_2 +0.10x_3 -1.05x_4)))^{-1}\\
&\qquad-0.45\times(1+\exp(-(1.82+4.78x_1 +0.48x_2-4.68x_3 -1.65x_4)))^{-1}\\
&\qquad-0.32\times(1+\exp(-(2.69+5.99x_1 +3.87x_2 +3.10x_3-2.17x_4)))^{-1},
\end{align*}

To measure the quality of a tablet, the same study introduced a score function $f_3$ (treated as a known function in our experiment), which combines the two properties $f_1$ and $f_2$: 
\begin{equation*}
    f_3 = \frac{(60-f_1)}{60}\times \frac{f_2}{1.5}, 
\end{equation*}
where the first term aims to ensure that the disintegration time is not too long (less than 60 seconds), and the second term aims to ensure that the tensile strength is large enough for production and distribution.

\subsection{Additional Experiment Without Upstream Evaluation Condition: Ackley6D+Matyas2D (AckMat)}
\label{app:addexp}
We consider the setting in which node evaluations do not require previously evaluated inputs from upstream nodes. Instead, we assume each node can be evaluated at any point in the set of possible outputs of the upstream nodes. We design this problem as a 7-dimensional cascade network where the first node is the 6-dimensional Ackley function \citepAP{ackley2012connectionist}:
\begin{equation*}
    f_1(x)=-20\exp\left(-0.2\sqrt{\frac{1}{6}\sum_{i=1}^6x_i^2}\right)-\exp\left(\frac{1}{6}\sum_{i=1}^6\cos{(2\pi x_i)}\right)+20+\exp(1),
\end{equation*}
where $x_i\in[-2,2]$ for $i=1,\hdots,6$. The second function node, which takes as input the output $y$ of the first node and one additional input $x_7$, is the negated Matyas function \citeAP{jamil2013literature}:
\begin{equation*}
    f_2(y,x_7)=-0.26(y^2+x_7^2)+0.48yx_7.
\end{equation*}
We set the range $x_7\in[-10,10]$ and we assume that the range of the output from the first node is known, i.e. $y\in[0,20]$.
The evaluation costs for this experiment are set to be $c_1 = 1$ and $c_2 = 49$ and we restrict to the same BO budget equal to 700. The results in Figure~\ref{fig:ackmatres} show that p-KGFN is also effective in this setting. Interestingly, we see that p-KGFN makes progress more slowly in the beginning, but then quickly overtakes and substantially outperform all baselines. This reflects the fact that the algorithm initially allocates most of its budget to learning about the behavior of the first node that is cheap to evaluate, and then with that knowledge moves to effectively optimize the second.
\begin{figure}[h]
    \centering
    \includegraphics[width=0.6\textwidth]{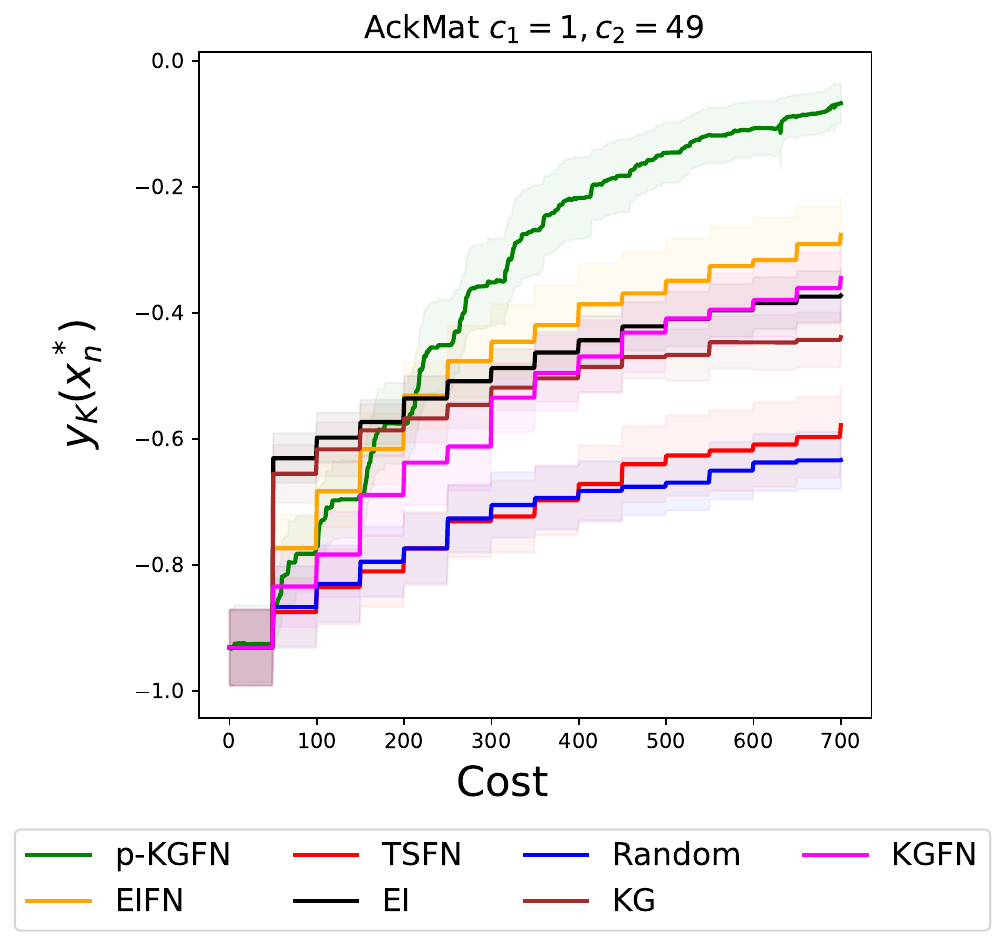}
    \caption{Optimization performance on AckMat problem without upstream evaluation condition comparing between our proposed p-KGFN and benchmarks including EIFN, KGFN, TSFN, EI, KG and Random.}
    \label{fig:ackmatres}
\end{figure}

\subsection{Additional Experiment with Noisy Observations} \label{appdx:addtl_exp_noisy}
In this section, we consider the FreeSolv problem presented in Section~\ref{sec:experiments} and Appendix~\ref{appx:freesolv}. We conduct additional experiments that add normally distributed noise to a node’s output before it is passed to subsequent nodes. 
We assume that the noise at each function node follows the standard normal distribution $\mathcal{N}(0,1)$. 

We use the noisy observations to update the GP describing each node. This entails standard equations for Gaussian process regression with noisy observations \citepAP{williams2006gaussian}. 

We consider our default setting, i.e. costs $c_1=1$ and $c_2=49$ with a total BO budget equal to 700. 
Figure~\ref{fig:resultnoise} illustrates the performance comparison between p-KGFN and benchmark algorithms on this variant of the test problem. The results demonstrate that p-KGFN still outperforms all benchmark algorithms, indicating its robustness to observation noise.

\begin{figure}
    \centering
    \includegraphics[width=0.4\textwidth]{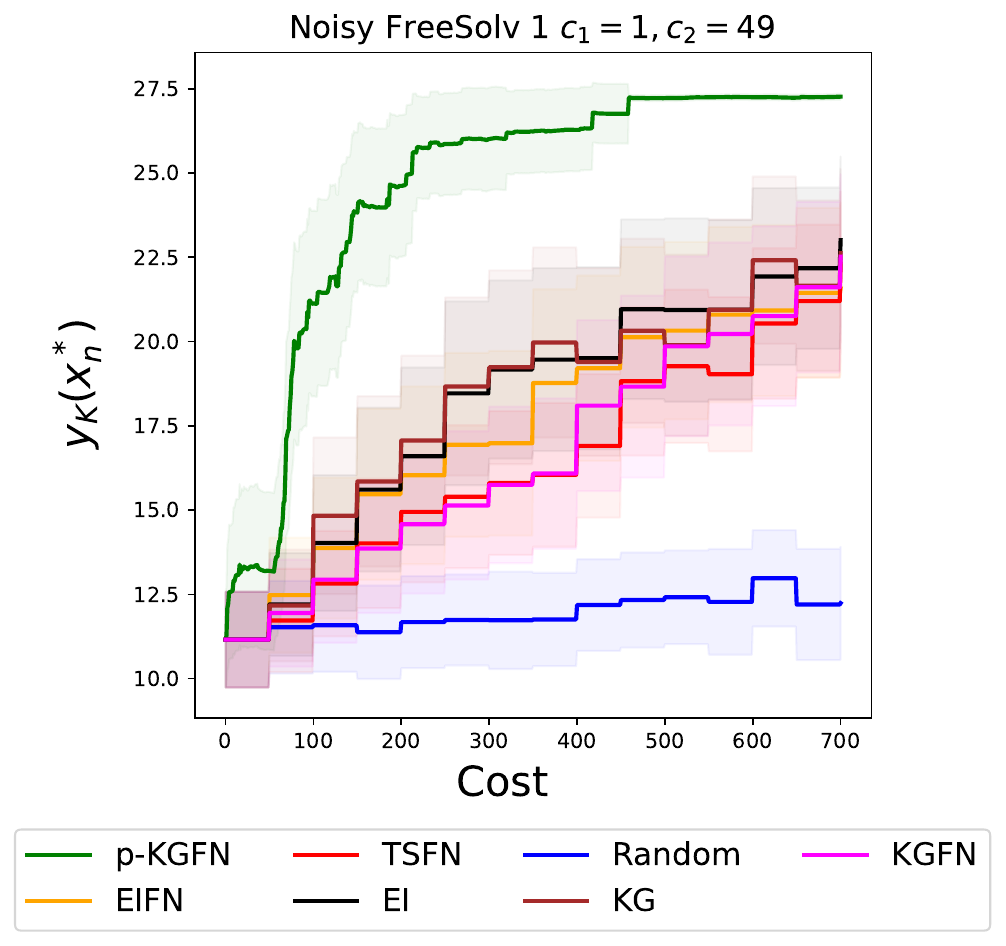}
    \caption{Optimization performance comparing between our proposed p-KGFN and benchmarks including EIFN, KGFN, TSFN, EI, KG, and Random on the FreeSolv problem with noisy observations, averaging over 30 trials.}
    \label{fig:resultnoise}
\end{figure}

\subsection{Additional Experiments where Downstream Nodes are More Difficult to Optimize}\label{appdx:addtl_exp_hard_downstream}
We conduct additional experiments where upstream nodes are more difficult to optimize than downstream nodes. We consider two problem setups:
\begin{itemize}
    \item (GPs-1): A sequential network with two nodes, shown in Figure~\ref{fig:eggandfree}. The first function node takes in a 1-D input, $x\in[-1,1]$. Both function nodes are drawn from GP priors. The lengthscales of the GP priors for the first and second nodes are set to 0.5 and 0.25, respectively. This ensures that the second node is more difficult to optimize compared to the first node. As with the test problems in our main paper, we set the cost of evaluating the second node to be substantially higher than that of evaluating the first node, i.e., $c_1 = 1$ and $c_2 =49$. We set the total BO budget at 700.
    \item (GPs-2): A function network with four function nodes (Figure~\ref{fig:GPs2}). The first three function nodes, $f_1$, $f_2$ and $f_3$, respectively take input $x_1$, $x_2, x_3\in[-1,1]$. The final function node, $f_4$, takes the outputs of $f_1$, $f_2$, and $f_3$ as its inputs. All functions are drawn from GP priors with a common lengthscale. We set the evaluation costs to be $c_1=c_2=c_3=1$ and $c_4=47$ and a total BO budget of 700.
\end{itemize}
We used the same settings for other parameters, such as the number of initial observations, as in the main experiments.
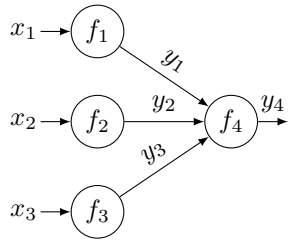
\begin{figure}[]
\centering
\begin{tikzpicture}[
init/.style={
  draw,
  circle,
  inner sep=0.7pt,
  minimum size=0.7cm
},
init2/.style={
  circle,
  inner sep=0.2pt,
  minimum size=0.2cm
},
]
\begin{scope}[start chain=1,node distance=4mm]
\node[on chain=1, init2]
(x1) {$\bmx_1$};
\node[on chain=1, init] 
  (f1) {$f_1$};
\end{scope}
\begin{scope}[start chain=2,node distance=4mm]
 \node[on chain=2, init2] at (0,-12mm)
(x2) {$\bmx_2$};
\node[on chain=2,init] 
 (f2) {$f_2$};

\node[on chain=2,init] at (2,-12mm)
 (f4) {$f_4$};
 \node[on chain=2,init2] (f5) 
  {};
 \end{scope}
\begin{scope}[start chain=3,node distance=4mm]
  \node[on chain=3, init2] at (0,-24mm)
(x3) {$\bmx_3$};
\node[on chain=3,init]
 (f3) {$f_3$};
  \end{scope}
\draw[-latex] (f1) -- (f4)node[pos=0.5,sloped,above] {$y_1$};
\draw[-latex] (f2) -- (f4)node[pos=0.5,sloped,above] {$y_2$};
\draw[-latex] (f3) -- (f4)node[pos=0.5,sloped,above] {$y_3$};
\draw[-latex] (f4) -- (f5)node[pos=0.5,sloped,above] {$y_4$};
\draw[-latex] (x1) -- (f1);
\draw[-latex] (x2) -- (f2);
\draw[-latex] (x3) -- (f3);
\end{tikzpicture}
\vspace{0.1in}
\caption{A function network structure for an additional experiment where the first layer node is harder-to-learn than the second layer (GPs-2).}
\label{fig:GPs2}
\end{figure}

As presented in Figure~\ref{fig:addexper}, our algorithm, p-KGFN, performs comparably to the other benchmarks in these additional problems.
\begin{figure}
    \centering   \includegraphics[width=0.7\linewidth]{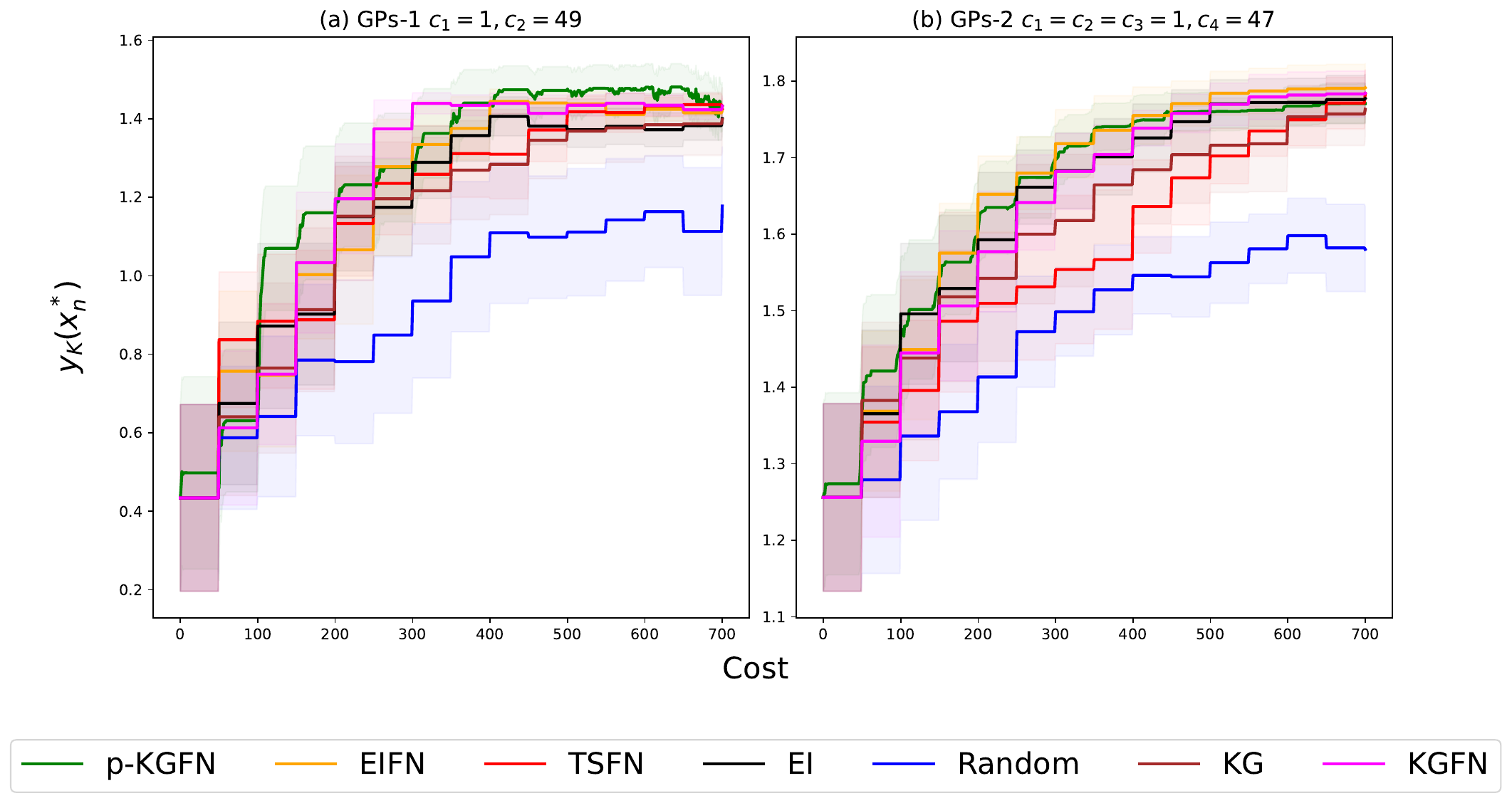}
    \caption{Optimization performance comparison between p-KGFN and benchmark algorithms on additional experiments where the first layer node is harder-to-learn than the second layer node. (a) GPs-1 with evaluation costs $c_1=1$, $c_2=49$ and (b) GPs-2 with evaluation costs $c_1=c_2=c_3=1$ and $c_4=47$. Both problems have BO budget equal to 700. Performance curves for p-KGFN and benchmarks, averaging over 30 replications. The mean and $\pm 2$ standard errors of the mean over the evaluation budget used are reported.}
    \label{fig:addexper}
\end{figure}

\section{Sensitivity Analysis for Evaluation Costs}\label{appdx:sensitivity}
We conduct a sensitivity analysis to examine the impact of cost functions on the optimization performance across three experiments with two nodes presented in the main text: Ackley (result is presented in the main text), FreeSolv and Pharm. In this section, we again consider scenarios where evaluating a downstream node requires previously obtained outputs from its parent nodes. This implies that the first node must be evaluated regardless of its cost. Our focus is thereby directed towards assessing the consequences of varying the cost associated with the second node. We investigate three cost function scenarios: (a) $c_1=1, c_2=1$; (b) $c_1=1, c_2=10$; and (c) $c_1=1, c_2=49$, which correspond to the situations where both nodes have similar evaluation costs, where one node has a higher evaluation cost than the other, and where one node has an exceptionally high evaluation cost, respectively. The evaluation budgets for each problem are set to 50, 150, and 700, respectively, in the three scenarios.

We also conduct the sensitivity analysis study for an additional experiment: AckMat presented in Appendix \ref{app:addexp} for which we do not impose the upstream evaluation restriction.

Figure~\ref{fig:cost} reveals that performing partial evaluations notably improves optimization performance, 
especially when the costs of the two nodes are dramatically different. On the other hand, in the equal-cost scenario, p-KGFN takes less advantage of partial evaluating, tending to complete full evaluations in sequential networks (Ackley and FreeSolv), and chooses to evaluate the two properties in Pharm problem an equal number of times. Results for AckMat problem are presented in the last row of Figure~\ref{fig:cost} and are consistent with the previous three problems with upstream evaluation restriction. Table~\ref{tab:nodecount} reports the average number of times p-KGFN selected to evaluate each node in each problem and cost scenario.

\begin{figure}
    \centering
    \includegraphics[width=0.9\textwidth]{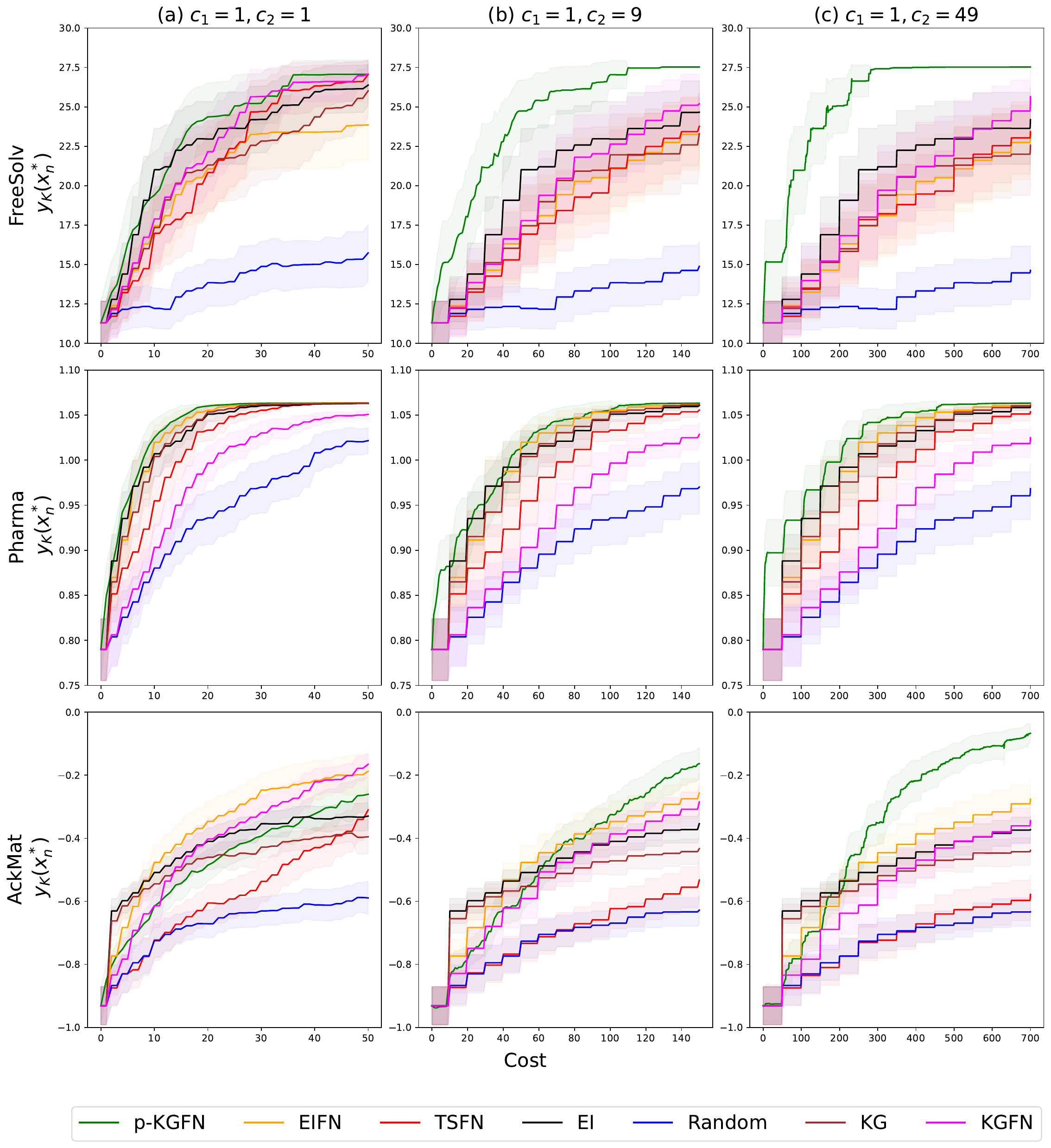}
    \caption{Cost sensitivity analysis for (from top to bottom) FreeSolv, Pharma and AckMat problems with different costs: (a) $c_1=1, c_2=1$; (b) $c_1=1, c_2=10$; and (c) $c_1=1, c_2=49$. The performance metric is the true objective value at the maximizer of final function node's posterior mean versus the budget spent.}
    \label{fig:cost}
\end{figure}

\begin{table}[h!]
\caption{Number of times each node was evaluated by p-KGFN and benchmark algorithms (averaging over 30 replications) for different problems and costs of evaluation.}
\label{tab:nodecount}
\resizebox{\textwidth}{!}{%
\begin{tabular}{lccc}
\toprule
\textbf{Problem} & \textbf{Costs of evaluation} & \textbf{\begin{tabular}[c]{@{}c@{}}Average number of times each \\ node was evaluated by p-KGFN\end{tabular}} & \textbf{\begin{tabular}[c]{@{}c@{}}Number of full function network \\ evaluations by benchmark algorithms\end{tabular}} \\ \midrule
Ackley & $c_1 = 1, c_2=1$ & $[36.1,13.9]$ & 25 \\
Ackley & $c_1 = 1, c_2=9$ & $[48.9,11.2]$ & 15 \\
Ackley & $c_1 = 1, c_2=49$ & $[64.6, 13.0]$ & 14 \\
Manufacturing & $c_1=5, c_2=10, c_3=10, c_4=45$ & $[30.8, 13.9, 18.0, 5.0]$ & 10 \\
FreeSolv & $c_1 = 1, c_2=1$ & $[31.0, 19.0]$ & 25 \\
FreeSolv & $c_1 = 1, c_2=9$ & $[39.0, 12.3]$ & 15 \\
FreeSolv & $c_1 = 1, c_2=49$ & $[62.3, 12.9]$ & 14 \\
Pharm & $c_1 = 1, c_2=1$ & $[23.8, 26.0]$ & 25 \\
Pharm & $c_1 = 1, c_2=9$ & $[27.6, 13.6]$ & 15 \\
Pharm & $c_1 = 1, c_2=49$ & $[63.0,13.0]$ & 14 \\
AckMat & $c_1 = 1, c_2=1$ & $[21.8, 28.2]$ & 25 \\
AckMat & $c_1 = 1, c_2=9$ & $[42.0, 12.0]$ & 15 \\
AckMat & $c_1 = 1, c_2=49$ & $[74.4, 12.8]$ & 14 \\ \bottomrule
\end{tabular}
}
\end{table}
\section{Wall Clock Times}
\label{supp:wallclock}
In this section, we report wall clock time on 8-core CPUs used to optimize each acquisition function on Ackley experiment.
\begin{table}[h!]
\caption{Acquisition optimization wall clock time in seconds on 8-core CPUs. Mean values and ± 2 standard errors are reported. KGFN takes significantly longer to optimize than p-KGFN because we use a larger number of samples when approximating its acquisition value.}\label{tab:wallclock}
\resizebox{\textwidth}{!}{%
\begin{tabular}{lccccccc}
\toprule
\textbf{Problem} & \textbf{EI} & \textbf{KG} & \textbf{Random} & \textbf{EIFN} & \textbf{KGFN} & \textbf{TSFN} & \textbf{p-KGFN} \\ \midrule
Ackley & 6.7 $\pm$ 0.5 & 76.9 $\pm$ 4.5 & 0.00033 $\pm$ 0.00001 & 51.9 $\pm$ 4.8 & 1362.6 $\pm$ 50.1 & 7.7 $\pm$ 0.2 & 246.6 $\pm$ 5.1 \\ 
Manufacturing & 4.5 $\pm$ 1.1 & 54.4 $\pm$ 7.8 & 0.00025 $\pm$ 0.00001 & 29.5 $\pm$ 3.6 & 2047.8 $\pm$ 83.1 & 4.0 $\pm$ 0.1 & 302.6 $\pm$ 15.0 \\ 
FreeSolv & 3.4 $\pm$ 0.3 & 111.7 $\pm$ 10.8 & 0.00036 $\pm$ 0.00001 & 57.9 $\pm$ 5.3 & 1050.9 $\pm$ 98.3 & 1.8 $\pm$ 0.1 & 158.7 $\pm$ 6.0 \\ 
Pharma & 3.9 $\pm$ 0.3 & 27.6 $\pm$ 1.4 & 0.00033 $\pm$ 0.00001 & 14.6 $\pm$ 1.4 & 222.4 $\pm$ 13.7 & 5.9 $\pm$ 0.7 & 101.4 $\pm$ 3.2 \\ 
AckMat & 1.4 $\pm$ 0.1 & 306.7 $\pm$ 27.0 & 0.00029 $\pm$ 0.00001 & 40.4 $\pm$ 2.4 & 1634.5 $\pm$ 86.2 & 9.3 $\pm$ 0.3 & 508.0 $\pm$ 16.7 \\ \bottomrule
\end{tabular}

}
\end{table}

\section{Additional Illustration of the Benefits of Partial Evaluations}
\label{app:pKGFNVsKGFN}
In this section, we add the performance of KGFN with full evaluations to our 1-dimensional illustration example previously presented in Section~\ref{subsec:advantage_partial} in order to highlight the substantial incremental benefits of performing partial evaluation. KGFN with full evaluations exhibits a similar behaviour to EIFN as depicted in the third row of Figure~\ref{fig:toyproblemwithFKGFN}. Specifically, KGFN makes decisions towards its goal of identifying a point with the best solution quality. It first decides to evaluate around the initial best inferred solution and then spends two full evaluations exploring the boundaries of the domain where uncertainty is high. Focusing only on the final goal without taking evaluation costs into account makes KGFN fail to obtain an accurate final composite function model and a good best inferred solution (pink square).
\begin{figure*}[h!]
    \centering
    \includegraphics[width=0.95\textwidth]{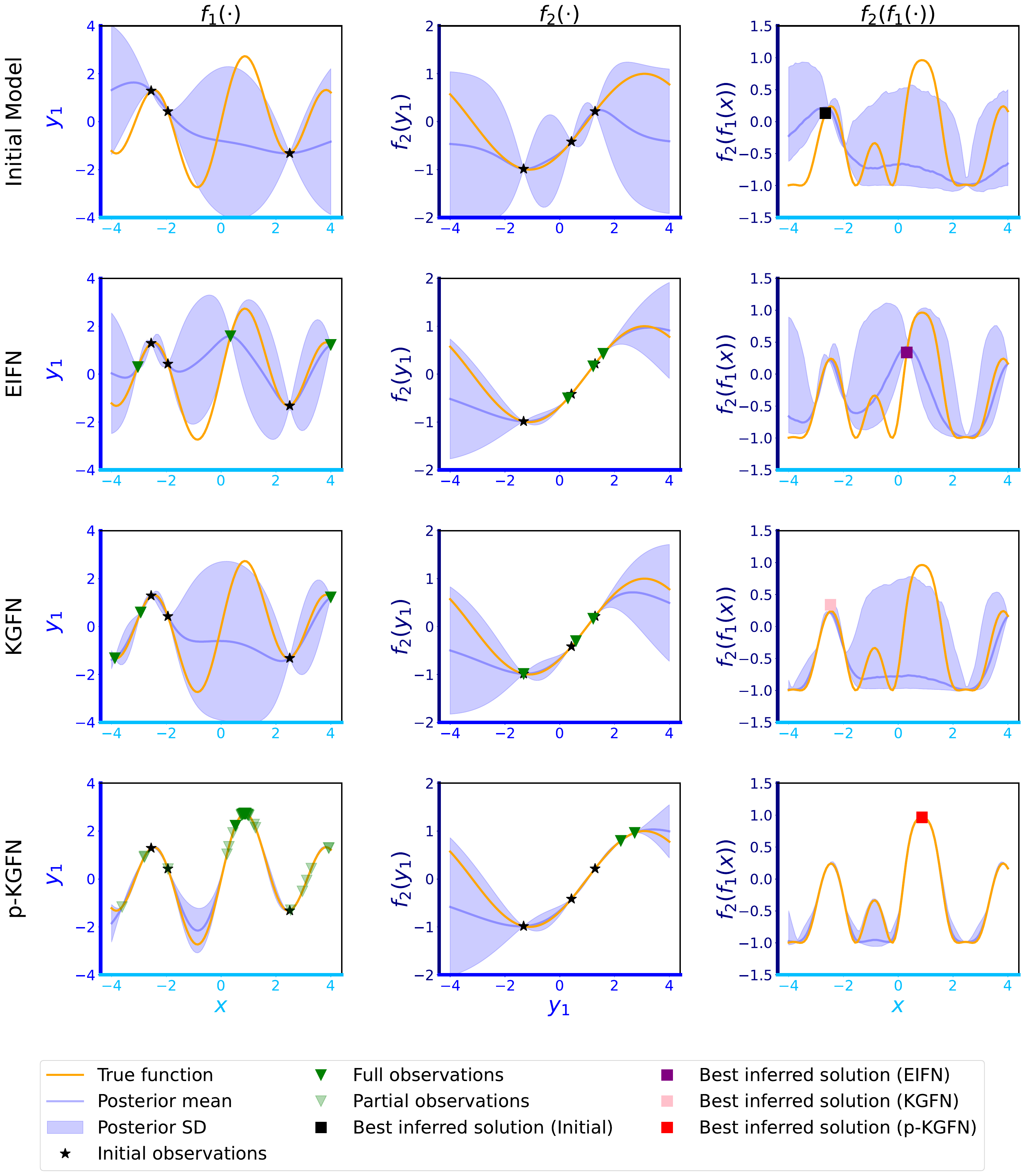}
    \caption{Comparison of EIFN, KGFN and p-KGFN on a 1-D synthetic two-stage function network $f_2(f_1(\cdot))$. The top row, from left to right, shows the initial models for $f_1(\cdot)$, $f_2(\cdot)$ and $f_2(f_1(\cdot))$. Similarly, the second, third and fourth rows show the resulting models upon budget depletion by EIFN, KGFN, and p-KGFN. Each true function is represented by an orange curve, while  blue curves and shaded blue areas denote posterior mean functions and posterior uncertainty, respectively. Black stars indicate the initial three points fully evaluated across the network for both algorithms. Dark green triangles represent the locations of full network evaluations. Light green triangles represent partial observations where only the first node was evaluated by p-KGFN. Black, purple, pink and red squares correspond to the initial and three final best inferred solutions identified by EIFN KGFN, and p-KGFN, respectively.}
    \label{fig:toyproblemwithFKGFN}
\end{figure*}

\section{Alternative Approach to Computing the  Comparison Metric}
\label{app:alter}
As outlined in the main text, we employed a posterior distribution of the final node value $y_K$ obtained from a statistical model that utilizes a network structure to compute our optimization comparison metric $y_K(x^*_n)$ across all algorithms including EI, KG and Random. The purpose is to underscore benefits of partial evaluations, but it unnecessarily favors these three algorithms as they do not actually consider a network structure in decision-making. In order to provide a more equitable comparison, we include the progress curves of the metric computed using a posterior distribution obtained from a standard Gaussian process model for these three algorithms. The results presented in Figure~\ref{fig:newmetric} illustrate, as expected, a degradation in their performance due to this modification. Notably, the Random baseline exhibits a declining trend in the AckMat problem when the network structure is not utilized. This is explained by the fact that the problem has a relatively small region of favorable outcomes.
\begin{figure}[h!]
    \centering
\includegraphics[width=\textwidth]{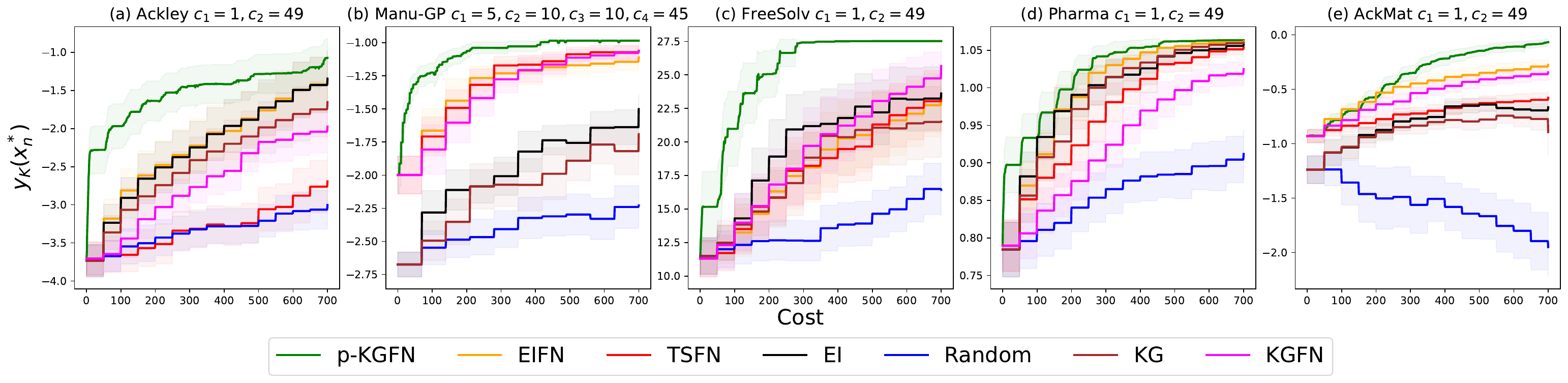}
    \caption{Optimization performance comparing between our proposed p-KGFN and benchmarks including EIFN, KGFN, TSFN, EI, KG and Random on four experiments: (a) Ackley, (b) Manu-GP, (c) FreeSolv, (d) Pharm and (e) AckMat. Every algorithm utilizes a statistical model in its decision-making process to calculate the comparison metric.}
    \label{fig:newmetric}
\end{figure}

\clearpage

\bibliographyAP{ref}
\bibliographystyleAP{icml2024}

\end{document}